\documentclass[conference]{IEEEtran}
\usepackage{times}

\usepackage[numbers]{natbib}
\usepackage{multicol}
\usepackage[bookmarks=true]{hyperref}

\usepackage{graphicx}

\usepackage{booktabs}
\usepackage{subfigure}
\usepackage{amsmath} 
\usepackage{amssymb}

\usepackage[T1]{fontenc}

\newtheorem{proposition}{Proposition}
\begin{document}

\title{Goal Discovery with Causal Capacity for Efficient Reinforcement Learning}

\author{\authorblockN{
Yan Yu\authorrefmark{1},
Yaodong Yang\authorrefmark{2},
Zhengbo Lu\authorrefmark{3}, 
Chengdong Ma\authorrefmark{4},
Wengang Zhou\authorrefmark{5},
Houqiang Li\authorrefmark{5}}
\authorblockA{\authorrefmark{1}University of Science and Technology of China, Email: yy1140730050@mail.ustc.edu.com}
\authorblockA{\authorrefmark{2}Institute for AI, Peking University,
Email: yaodong.yang@pku.edu.cn}
\authorblockA{\authorrefmark{3}Institute of Artificial Intelligence,
Email: luzhenbo@iai.ustc.edu.cn}
\authorblockA{\authorrefmark{4}Institute for AI, Peking University,
Email: mcd1619@buaa.edu.cn}
\authorblockA{\authorrefmark{5}University of Science and Technology of China,
Email: \{zhwg,lihq\}@ustc.edu.cn}
}

\maketitle

\begin{abstract}
  Causal inference is crucial for humans to explore the world, which can be modeled to enable an agent to efficiently explore the environment in reinforcement learning.
  Existing research indicates that establishing the causality between action and state transition will enhance an agent to reason how a policy affects its future trajectory, thereby promoting directed exploration.
  However, it is challenging to measure the causality due to its intractability in the vast state-action space of complex scenarios.
  In this paper, we propose a novel \textbf{G}oal \textbf{D}iscovery with \textbf{C}ausal \textbf{C}apacity (GDCC) framework for efficient environment exploration. 
  Specifically, we first derive a measurement of causality in state space, \emph{i.e.,} causal capacity, which represents the highest influence of an agent's behavior on future trajectories.
  After that, we present a Monte Carlo based method to identify critical points in discrete state space and further optimize this method for continuous high-dimensional environments.
  Those critical points are used to uncover where the agent makes important decisions in the environment, which are then regarded as our subgoals to guide the agent to make exploration more purposefully and efficiently.
  Empirical results from multi-objective tasks demonstrate that states with high causal capacity align with our expected subgoals, and our GDCC achieves significant success rate improvements compared to baselines. 
\end{abstract}

\IEEEpeerreviewmaketitle

\section{Introduction}
Reinforcement Learning (RL) has proven to be an effective approach for training agents to perform a wide range of tasks, achieving notable success in domains such as games \cite{berner2019dota}, autonomous driving \cite{Xiao_Li_Wang_Peng_Wu_Zhao_Zhang_2023}, and robotics \cite{kalashnikov2018scalable}. 
In RL, an agent explores the environment, gathers data and maximizes accumulated reward to learn a high-quality policy. 
Generally, there exists a causal association between an agent's action and future trajectory. 
Modeling such causal association will enhance the agent's ability to explore and exploit the environment, which leads to efficient policy learning. 

\begin{figure}[t]
\centering 
\subfigure[Illustration scenarios]{
    \label{Instance}
    \includegraphics[width=0.45\columnwidth]{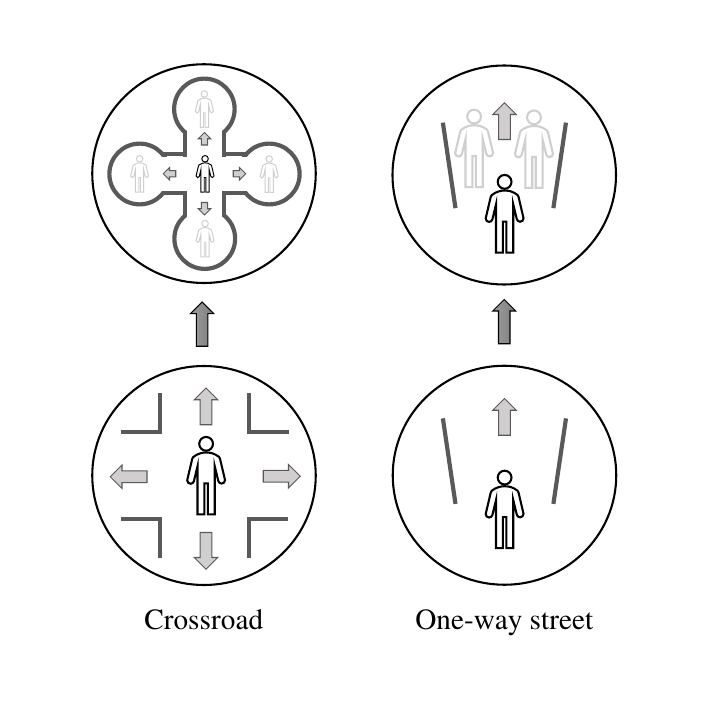}
    }
\subfigure[Causal capacity]{
    \label{DEMO}
    \includegraphics[width=0.48\columnwidth]{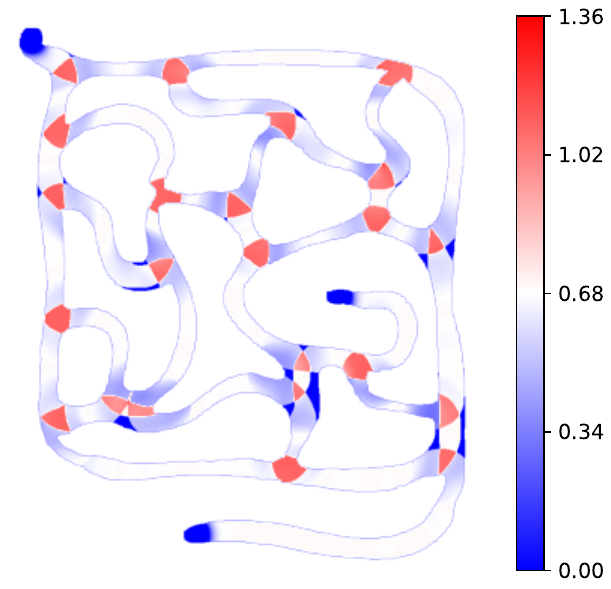}
    }
    \caption{(a) Two scenarios demonstrates that actions taken at different states will have varying impacts on the future. At the crossroads, an agent's different choices will result in different destinations, while on a one-way street, shifts to the left or right have little influence on the final destination. (b) The causal capacity results of all states in a demo maze environment, where red regions represent high causal capacity and blue regions indicate low causal capacity. It is observed that causal capacity effectively highlights the states where the agent can exert control over its future trajectory.}
    \label{fig:illustration}
\end{figure}
Causality plays a crucial role in decision-making. 
Humans typically desire that the outcomes of their actions align with their active decisions rather than being passively driven by the environment. 
Decisions made at critical points often have significant causal impacts on future outcomes. 
For instance, at a crossroad, the chosen direction may significantly affect the final destination, whereas shifts on a one-way road result in  only trivial changes in the final destination. 
For agents, learning actions with strong causal effects on the future is more valuable. 
The causal value of an action depends not only on the action itself but also on the state in which it is executed, similar to the example in Fig. \ref{Instance}. 
Existing research has focused on directly measuring the causal association between actions and state transitions,  enabling more comprehensive exploration \cite{seitzer2021causal}. 
However, due to the vastness of the state-action space, such measurement is complex and inaccurate. 
To address this, we resort to identify critical points in the environment like the red regions in Fig. \ref{DEMO}, where an agent's actions have a clear causal associations with the expected destination, facilitating more efficient exploration and policy optimization.

In this work, we propose a causality-aware framework that enables the agent to understand the association between states and actions. 
Drawing from the concept of maximum caliber in statistical physics \cite{Jaynes1957InformationTA,Dixit_Wagoner_Weistuch_Pressé_Ghosh_Dill_2018,Chvykov_Berrueta_Vardhan_Savoie_Samland_Murphey_Wiesenfeld_Goldman_England_2021}, we derive the causal capacity from Granger entropy \cite{Granger}. 
Causal capacity measures the maximum causal impact of the agent's actions on future trajectories.
It is defined as the entropy of the probability distribution over the state transition. 
Essentially, it measures the uncertainty in state transitions. 
A state with a larger causal capacity indicates more available choices for the agent.
To address the challenge of accurately measuring causal capacity, we propose a Monte Carlo-based method that requires only data collected through a random policy, allowing us to effectively measure the causal capacity of each state.

Based on the measurement of causal capacity, we can identify critical points in the environment that align with our expectations, selecting them as subgoals. 
By achieving these subgoals sequentially, the agent can explore the environment more effectively and train more efficiently.
To further utilize the subgoals during training, we preserve the sequential structure of the random policy data and train a prediction model along with a directed acyclic graph (DAG). 
This clarifies the causal associations among subgoals and improves the training efficiency and effectiveness of downstream tasks.

To evaluate our approach, we design multi-objective tasks in the MuJoCo maze environment \cite{2012MuJoCo} and the Habitat environment \cite{Savva_Kadian_Maksymets_Zhao_Wijmans_Jain_Straub_Liu_Koltun_Malik_etal,Szot_Clegg_Undersander_Wijmans_Zhao_Turner_Maestre_Mukadam_Chaplot_Maksymets_etal,Puig_Undersander_Szot_Cote_Yang_Partsey_Desai_Clegg_Hlavac_Min_etal}. 
In these tasks, the agent cannot simply memorize paths to the goal, but must understand the association between the environment and the task. 
Our empirical results demonstrate that the calculated subgoals align perfectly with our expectations. 
Furthermore, our method outperforms baseline algorithms, demonstrating the effectiveness of the GDCC framework in subgoal exploration.

\section{Related Work}
\subsection{Causal Reinforcement Learning}
Causal reinforcement learning \cite{Deng_Jiang_Long_Zhang_2023} aims to develop agents capable of comprehending their environment, solving complex tasks and improving the interpretability of decision-making processes. 
Many previous works focus on the advantages of CRL in task generalization, discovery of spurious correlations, representation learning, and data augmentation \cite{Zhang_Lyle_Sodhani_Filos_Kwiatkowska_Pineau_Gal_Precup_2020,diaz2013sensitivity,Swamy_Choudhury_Bagnell_Wu,sontakke2021causal,Buesing_Weber_Zwols_Racanière_Guez_Lespiau_Heess_2018}. 
Besides, encouraging the agent to discover the causal mechanisms underlying state transitions is crucial for facilitating exploration.
In~\cite{NEURIPS2022_7e9fbd01}, it models the causalities among environments variables (EV) to discover subgoals and high-quality hierarchical structures in complicated environment.
In~\cite{NEURIPS2022_a96368eb}, the problem is formulated into variational likelihood maximization with causal graph (CG) as latent variables. 
But both EV and CG requires strong prior knowledge about the environment.

On the other hand, conditional mutual information have been introduced as intrinsic reward to encourage agent to explore more diversely \cite{Eysenbach_Gupta_Ibarz_Levine_2018} or to detect the states of influence \cite{seitzer2021causal}.
However, it remains a significant challenge to accurately estimate the causality of the agent's behavior.
In this work, we aim to figure out an accurate measurement of the causality between the agent's behavior and the environment without relying on prior knowledge, so that the agent's policy can maximize its causal value. 
To achieve this, we employ causal discovery to identify subgoals, improving exploration and sampling efficiency.

\subsection{Goal-Conditioned Reinforcement Learning}
In the paradigm of Goal-Conditioned Reinforcement Learning (GCRL) \cite{NIPS1992_d14220ee,Kaelbling1993LearningTA}, complex tasks are decomposed into simpler tasks through subgoals and completed sequentially, similar to the problem-solving approach used by humans \cite{McGovern_Barto_2001,Bagaria_Konidaris_2020}. 
However, it is a non-trivial issue to generate subgoals in GCRL. 
Hindsight experience replay \cite{Andrychowicz_Wolski_Ray_Schneider_Fong_Welinder_McGrew_Tobin_Abbeel_Zaremba_2017} relabels achieved goals in the buffer as desired goals to better utilize data \cite{fang2019curriculum,bai2019guided}.
In~\cite{Chane-Sane_Schmid_Laptev_2021}, a value function is used to evaluate experience and selects intermediate states between the current state and the final goal as subgoals, optimizing the policy for selecting subgoals. 
However, these subgoals lack clear physical significance, and do not guarantee effective guidance in complex environments and tasks.
Other approaches encourage the agent to learn multiple effective skills \cite{simsek_Wolfe_Barto_2005,Mannor_Menache_Hoze_Klein_2004,Gehring_Synnaeve_Krause_Usunier_2021}, which, however, requires strong prior knowledge of the environment or are limited to simple, discrete environments.
In contrast, our method does not rely on any prior knowledge or expert data and can be applied in continuous state spaces. 
Our approach generates subgoals with causal significance through pretraining using the Monte Carlo method, without requiring high-quality offline data, relying instead on data sampled from a random policy.

\section{Preliminaries}
In this section, we provide the background knowledge of our method. We start by introducing the problem formulation of goal-conditioned reinforcement
learning, followed by detailed explanations of Structural Causal Models, Granger causality, and Transfer Entropy. These concepts serve as the theoretical foundations of our work.

\subsection{Problem Formulation}
The problem studied in this work is formulated as a goal-conditioned Markov Decision Process (MDP), represented as a six-tuple $\mathcal{M} =\left\langle\mathcal{S}, \mathcal{A}, \mathcal{G}, P, R, \gamma\right\rangle$. 
It includes a state space $\mathcal{S}$, an action space $\mathcal{A}$, and a subgoal space $\mathcal{G}$. 
The transition probability function $P$ defines the environment's intrinsic dynamics and is given by the conditional probability $p(s'\mid s, a)$. 
The reward function $R$ provides rewards based on the current state $s$, action $a$, subgoal $g$, and next state $s'$,  expressed as $r(s,a,g,s')$. $\gamma \in (0,1)$ is a discount factor. Our objective is to obtain an optimal policy $\pi:\mathcal{S},\mathcal{G}\rightarrow \mathcal{A}$ that maximizes the expected cumulative discounted reward $\mathbb{E}_\pi \left[\sum_{t=0}^\infty \gamma^t r_t \mid \pi \right]$. In this work, we focus on the sparse reward setting, where the agent receives zero rewards for most of the time.

\subsection{Structural Causal Model}

We use a Structural Causal Model (SCM) to represent the state transition in an MDP. 
As shown in Fig.~\ref{fig:SCG}, $U_{s_t}$ and $U_{a_t}$ represent independent noises or other unobserved confounders of the environment on state and action, respectively. 
The SCM consists of a set of random variables denoted by 
$\mathcal{V}=\{(S_i,U_{s_i},A_i, U_{a_i})\}_{i=1}^N$, and is structured using a directed acyclic graph (DAG).
Each node in the graph follows a conditional probability distribution $P(V_i\mid \text{Pa}(V_i))$, where $\text{Pa}(V_i)$ is the set of parents of $V_i$.

The causal model not only provides a clear structure of the variables but also facilitates the modeling of causal interventions \cite{stone2020causal}. 
In reinforcement learning, an agent's policy serves as an intervention mechanism, denoted as $I=\text{do}(A:=\pi(a\mid s))$, where the $\text{do}(\cdot)$ operator specifies fixing the value of a variable in the intervention process. 
Generally, intervening with different policies results in different distributions of the next state, \emph{i.e.,} $P^{\text{do}(A:=\pi_i(a\mid s))}(S'\mid S)\neq P^{\text{do}(A:=\pi_j(a\mid s))}(S'\mid S)$, but state transition remains the same, \emph{i.e.,} $p^{\text{do}(A:=\pi_i(a\mid s))}(s'\mid s, a) = p^{\text{do}(A:=\pi_j(a\mid s))}(s'\mid s, a)$.

\begin{figure}
    \centering
    \includegraphics[width=0.85\columnwidth]{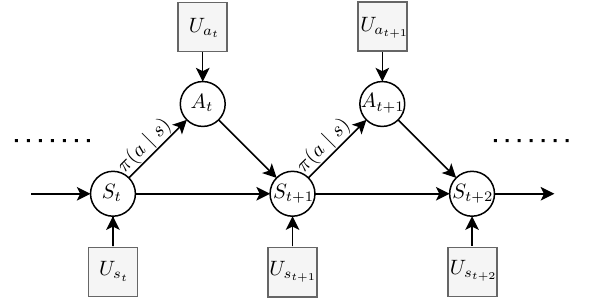}
    \caption{The structural causal model (SCM) illustrates the state transition from $S_t$ to $S_{t+1}$ as $S_{t+1}=f(S_t,A_t,U_{s_{t+1}})$. In the SCM, the policy \(\pi\) is highlighted as the causal intervention mechanism in reinforcement learning.
    }
    \label{fig:SCG}
\end{figure}

\subsection{Granger Causality and Transfer Entropy}
According to Granger causality \cite{Granger}, if including past information of variables $\mathbf{X}$ and $\mathbf{Y}$ helps improves the prediction of $\mathbf{Y}'$ compared to predicting $\mathbf{Y}'$ based solely on its past information, then ``$\mathbf{X}$ Granger-causes $\mathbf{Y}$ in Granger causality". Transfer entropy \cite{Schreiber_2002} follows a similar concept, measuring the directed information transfer between joint processes. For Gaussian variables, Granger causality and transfer entropy are equivalent \cite{Barnett_Barrett_Seth_2009}. The expression of transfer entropy is as follows:
\begin{equation}
\mathcal{T}(\mathbf{X}\rightarrow \mathbf{Y}) = \mathcal{H}(\mathbf{Y}'\mid \mathbf{Y})-\mathcal{H}(\mathbf{Y}'\mid \mathbf{Y},\mathbf{X}).
\end{equation}
Transfer entropy measures the degree to which $\mathbf{X}$ reduces uncertainty in predicting the future of $\mathbf{Y}$. Since its introduction \cite{Schreiber_2002}, transfer entropy has become a widely recognized tool for analyzing causal relationships in nonlinear systems \cite{HLAVACKOVASCHINDLER_PALUS_VEJMELKA_BHATTACHARYA_2007}.

\section{Method}
In this section, we introduce the derivation of causal capacity and the details of the GDCC framework, which includes subgoal generation and prediction. 
First, we analyze the causal association between the agent's actions and their outcomes, deriving the action causality measurement and defining causal capacity based on transfer entropy.
Given the challenges of measuring causal capacity without prior knowledge of the environment's dynamics, we design a Monte Carlo method to estimate causal capacity and further optimize this estimation for continuous, high-dimensional environments using a clustering algorithm.
Next, we identify the critical points with the highest causal capacity, which are used to guide the agent in purposeful exploration.
Finally, we propose a prediction model to obtain the optimal subgoal for the current state, effectively simplifying the task and reducing the exploration space.
The overall framework is shown in Fig.~\ref{fig:Overall}.

\begin{figure}[htb]
    \centering
    \includegraphics[width=0.95\columnwidth]{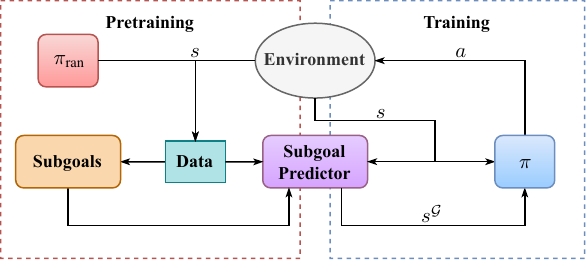}
    \caption{Overall framework of GDCC.}
    \label{fig:Overall}
\end{figure}

\subsection{Action Causality Measurement}
If performing action $a$ in state $s$ reduces the uncertainty of transition from $s$ to next state $s'$, it implies that there exists a state $s_i$ or a set of states $\mathbf{S}_i$ with a higher probability of transitioning from $s$ to $s_i$ or $\mathbf{S}_i$. It is reasonable to assume that there is a causal association between executing the action $a$ and the increased transition probability towards $s_i$ or $\mathbf{S}_i$, which can be quantified using transfer entropy as follows:
\begin{equation}
\begin{split}
\mathcal{T}(A\rightarrow &S\mid {S=s,\text{do}(A=a)})=\\
&\mathcal{H}(S'\mid S=s)-\mathcal{H}(S'\mid S=s,\text{do}(A=a)),
\end{split}
\end{equation}
where $\mathcal{H}(S'\mid S=s)$ denotes the entropy of the non-interventional state transition distribution of state $s$, and $\mathcal{H}(S'\mid S=s,\text{do}(A=a))$ represents the entropy of the state transition distribution with the agent's action set to \( a \).
\subsection{Causal Capacity Measurement}
We are interested in identifying the state that maximizes the diversity of the state transition while minimizing its uncertainty after taking actions. However, since directly measuring action causality requires estimating the state transition probabilities of all actions executed in all states, which is computationally infeasible in continuous state and action spaces, we need another variable to evaluate the causality. We propose the following propositions regarding the upper and lower bounds of action causality:

\begin{proposition}
\label{prop:bound}
    If the entropy of the non-interventional state transition distribution of state $s$ can be calculated, then the upper and lower bounds of the transfer entropy for any action $a$ are given as follows:
\begin{equation}
\begin{aligned}
\mathcal{H}(S'\mid S=s)&\geq\mathcal{T}(A\rightarrow S\mid {S=s,\rm{do}(A=a)})\\
&\geq \min_{a}\left(1-\frac{1}{p(a\mid s)}\right)\mathcal{H}(S'\mid S=s).
\label{Eq:Proposition-1}
\end{aligned}
\end{equation}
\end{proposition}

Furthermore, when we control the policy such that the agent selects actions with the most causal impact, the upper and lower bounds of action causality are:

\begin{proposition}
\label{prop:max-bound}
    When we choose an action to maximize the transfer entropy, and $\mathcal{H}(S'\mid S)$ can be calculated, the upper and lower bounds of the transfer entropy are as follows:
\begin{equation}
\mathcal{H}(S'|S=s)\geq\max_{a}\mathcal{T}(A\rightarrow S\mid {S=s,\rm{do}(A=a)})\geq0.
\label{Eq:Proposition-2}
\end{equation}
\end{proposition}
\noindent The proof of the proposition can be found in the supplementary material.

According to Eq.~\eqref{Eq:Proposition-1} and Eq.~\eqref{Eq:Proposition-2}, $\mathcal{H}(S'\mid S)$ represents the maximum potential causal influence that an agent’s action can have in state $s$. Meanwhile, under a controlled policy, the lower bound of action causality becomes independent of $\mathcal{H}(S'\mid S)$. Therefore, it is reasonable to use the entropy of the non-interventional state transition distribution as a measure of the maximum causal value of a state. We define the causal capacity as follows:
\begin{equation}
\mathcal{C}(s) = \mathcal{H}(S'\mid S=s).
\end{equation}

Through estimating the causal capacity of each state, we can identify states where the agent's actions have a substantial impact on the next state. 
However, accurately estimating causal capacity remains challenging, especially when the transition function $P$ is not fully understood.
Predicting the entire state transition distribution without action constraints is not feasible, even within the framework of Model-Based Reinforcement Learning (MBRL). The typical approach is to predict the next state based on a specific state-action pair, rather than relying solely on the state.

To overcome this challenge, we propose a Monte Carlo-based method for measuring causal capacity, making it applicable to real-world problems. Additionally, we incorporate a clustering algorithm to extend GDCC to continuous, high-dimensional environments.

\subsubsection{Monte Carlo Based Measurement}
Based on the definition of entropy, the causal capacity can be factorized as follows:
\begin{equation}
\mathcal{C}(s)=-\sum_{s_i \in S'}p(s_i\mid s)\log p(s_i\mid s).
\end{equation}
Since calculating $p(s'\mid s)$ requires knowledge of the state transition distribution under all actions, we propose using a Monte Carlo method to estimate it. To this end, we introduce the following proposition.  
The proof of this proposition can be found in the supplementary material.

\begin{proposition}
\label{prop:random}
The non-interventional state transition probability is equivalent to the state transition probability under the intervention of a random policy:
    \begin{equation}
        p(s' \mid s)=p^{\rm{do}(A:=\pi_{\text{ran}})}(s'\mid s),
    \end{equation}
    where $\pi_{\text{ran}}$ denotes the random policy. 
\end{proposition}
Therefore, it is reasonable to collect trajectory data $\mathcal{D}$ with a random policy $\pi_{\text{ran}}$ to approximate the non-interventional transition probabilities. 
We count the frequencies of each state $N(S=s)$ in $\mathcal{D}$. 
Accordingly, the causal capacity can then be represented as:  
\begin{equation}
\mathcal{C}(s)= -\sum_{s_i}p^{\text{do}(A:=\pi_{\text{ran}})}(s_i\mid s)\log p^{\text{do}(A:=\pi_{\text{ran}})}(s_i\mid s).
\end{equation}
where $p^{\text{do}(A:=\pi_{\text{ran}})}(s_i\mid s)\approx\frac{N(S'=s_i\mid S=s)}{N(S=s)}$.

\subsubsection{Clustering Based Measurement}
In complex environments or real-world scenarios, the state space is often high-dimensional and continuous. As a result, a state $s$ may only be visited once, making frequency-based statistics ineffective. 

Considering continuous state spaces, a small change in state, such as $s+\Delta s$, can be nearly identical to the original state $s$ in the physical world. In our study of causality, our primary objective is to determine whether the agent's actions will have significant impact on its future, which requires a clear difference in the physical world. 
Therefore, we cannot rely solely on state-action transitions as the basis for estimating the set of possible next states $S'$ for $s$. Instead, we propose partitioning the sampled state data $\mathbf{S}_{\mathcal{D}}$ into distinct state sets, based on $s$, as follows:
\begin{equation}
\left\{
\begin{aligned}
    &\mathbf{S}_{\text{nei}}(s)=\{\hat{s}\mid d(s,\hat{s})<\tau_{\text{nei}}\}, \\
    &\mathbf{S}_{\text{adj}}(s)=\{\hat{s}\mid \tau_{\text{nei}}\leq d(s,\hat{s})<\tau_{\text{adj}}\}, \\
    &\mathbf{S}_{\text{out}}(s)=\{\hat{s}\mid \tau_{\text{adj}}\leq d(s,\hat{s})\},
\end{aligned}
\right. 
\end{equation}
where $d(\cdot,\cdot)$ is a distance function, which could be a standard distance function, such as Euclidean distance or Manhattan distance, or a neural network trained on specific metrics. 
$\tau$s are the distance threshold set based on the scope of agent's actions. 
The set of neighboring states $\mathbf{S}_{\text{nei}}(s)$ includes those states within a distance of $\tau_{\text{nei}}$, which are considered to be in the same physical state with $s$. 
$\mathbf{S}_{\text{out}}(s)$ includes states that are too far away from $s$ to be reached by a single action and can only be accessed after multiple state transitions. 
$\mathbf{S}_{\text{adj}}$ denotes the expected set of next states for $s$, consisting of states that are adjacent and reachable in a single action.
We define the general set of next states for \(s\) as $
\tilde{\mathbf{S}}'(s)=\mathbf{S}_{\text{adj}}(s)$.

To estimate the probability distribution of different state transitions in $\tilde{\mathbf{S}}'(s)$, we measure the distance between each pair of states in $\tilde{\mathbf{S}}'(s)$ using $d(\cdot, \cdot)$, and then apply the Agglomerative Clustering algorithm \cite{sklearn_api}. This process partitions $\tilde{\mathbf{S}}'(s)$ into $N$ clusters, \emph{i.e.,} $\textbf{Cluster}(\tilde{\mathbf{S}}'(s))=\{\tilde{\mathbf{S}}_1,\tilde{\mathbf{S}}_2,\ldots,\tilde{\mathbf{S}}_N\}$. We then use the frequency $|\tilde{\mathbf{S}}_i|$ of each state cluster to approximate the probability of each state transition type. Based on this clustering, we can calculate the causal capacity for the state \(s\):
\begin{equation}
\mathcal{C}_{\text{clu}}(s)=\sum_{\mathbf{\tilde{S}}_i\in \mathbf{\tilde{S}}'(s)}p_{\text{clu}}(\mathbf{\tilde{S}}_i\mid s)\log p_{\text{clu}}(\mathbf{\tilde{S}}_i\mid s),
\end{equation}
where $p_{\text{clu}}(\mathbf{\tilde{S}}_i\mid s)=\frac{|
\mathbf{\tilde{S}}_i|}{|\mathbf{\tilde{S}}'(s)|}$.

The maximum causal capacity of a state $s$ is relative to the number of clusters of $\tilde{\mathbf{S}}'(s)$. States with a small number of clusters are constrained, preventing the agent from making its own choices, or they are situated at a larger state in physics where agent cannot transition to another state in a single step.

The primary reason to use clustering algorithms is that predicting the distribution without intervention requires fully sampling the state transition under all actions and then estimating the distribution for each next state. 
This contrasts with typical Model-Based Reinforcement Learning (MBRL) tasks, where the goal is often to predict the mean of the next states. 
In our case, however, we are particularly interested in accurately estimating the variance of the next state distribution.
The standard MBRL approaches for estimating variance in continuous state spaces may not always meet our requirement.
Therefore, we adopt a distance-based statistical method combined with a clustering algorithm.
Additionally, we can design a distance function $d(\cdot, \cdot)$ to characterize the environment based on state representation. 
By incorporating more information about the environment, including temporal and semantic information, we improve the representation and distinguishability of clustering algorithm in state transitions. Empirical results of this clustering approach can be found in the supplementary material.

\subsection{Subgoal Prediction}
The purpose of calculating the causal capacity is to find the most suitable subgoals in the environment. Once the causal capacity of each state has been computed, we can select those states whose causal capacity exceeds a certain threshold as subgoals. By constraining the agent's actions to these subgoal states, we maximize the likelihood that its future trajectory will be controlled and lead to the desired outcomes.

However, in some cases, the agent may not be able to explore the entire environment through random policy. This could prevent us from obtaining the causal capacity of all states. In such scenarios, we employ the Go-Explore approach \cite{Ecoffet_Huizinga_Lehman_Stanley_Clune_2019}. It involves training a model to achieve the latest subgoal and then exploring with random policy to complete the exploration of the entire environment.

Once the agent has executed actions in the environment, the next challenge is how it can select the optimal subgoal for any given state. To address this, we propose a prediction model that identifies the most suitable subgoal for each state. Its structure is shown in Fig. \ref{fig:high_level}. The prediction model consists of two key components: (1) an encoder and a decoder are self-supervised pretrained for embedding states and distinguishing subgoals, (2) a predictor for subgoal prediction. The encoder $p_{\theta}(z\mid s)$ takes states $s$ and subgoals $s^{\mathcal{G}}$ as input, projecting them into latent space as $z$ and $z^{\mathcal{G}}$. The decoder $q_{\phi}(s\mid z)$ reconstructs the embedded states and subgoals back to original space $s'$ and ${s'}^{\mathcal{G}}$. During encoding and decoding, the encoder also minimizes the similarity between each pair of subgoals in the latent space, ensuring that subgoals remain distinguishable while preserving information from the original state. The loss functions for the encoder and decoder are defined as follows:
\begin{equation}
    \mathcal{L}(\theta,\phi)=\lambda_{\theta} \sum_{s_i\in \mathcal{D}}\left \|s_i-s_i'\right \|^2 +\lambda_{\phi} \sum_{i\neq j} \text{sim}(z^{\mathcal{G}}_{i},z^{\mathcal{G}}_{j}),
\end{equation}
where $\text{sim}(\cdot, \cdot)$ represents the similarity function, commonly using the cosine similarity measurement. 
Both $\lambda_{\theta}$ and $\lambda_{\phi}$ are positive coefficients.

\begin{figure}
    \centering
    \includegraphics[width=\columnwidth]{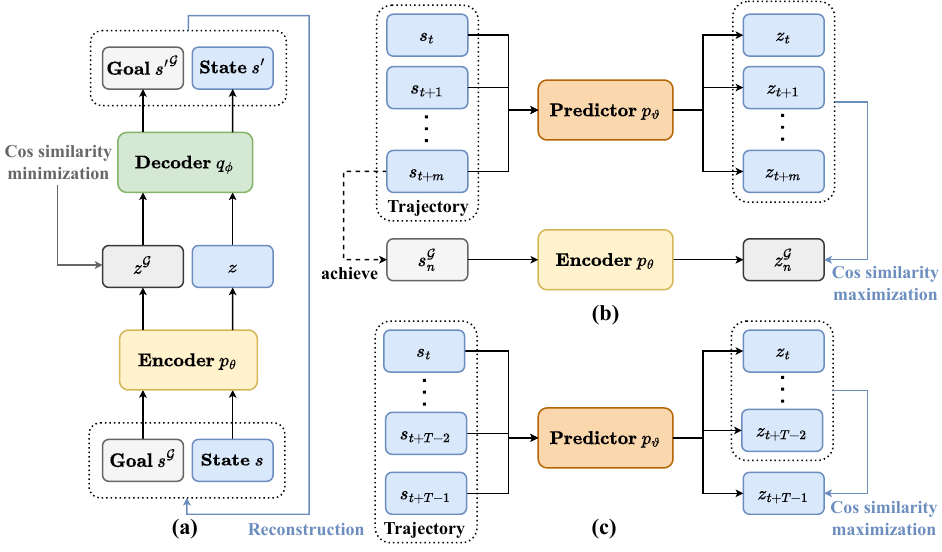}
    \caption{Illustration of subgoal prediction model. (a) Pretraining encoder and decoder. (b) Training predictor when $s_m$ achieves subgoal $s_n^{\mathcal{G}}$. (c) Training predictor when no subgoal is achieved.}
    \label{fig:high_level}
\end{figure}

Next, we train the predictor $\rho_\vartheta (z\mid s)$ for subgoal prediction. We sample sequential trajectory $\tau=\{s_t,s_{t+1},\cdots,s_{t+T-1}\}$ of length $T$ and check if there is any state in $\tau$ achieves subgoal. 
If a state $s_m$ achieves the subgoal $s^{\mathcal{G}}_n$, the expected prediction target for the states before $s_{t+m}$ is set to $p_\theta(s^{\mathcal{G}}_n)$. If not, the prediction result of the last state $\rho_{\vartheta}(s_{t+T-1})$ will be set as the expected prediction target for the entire trajectory. The loss function for the predictor is defined as follows:
\begin{equation}
    \mathcal{L}(\vartheta)=
\left\{\begin{aligned}
    &-\frac{1}{m+1}\sum_{i=t}^{t+m} \text{sim}(\rho_{\vartheta}(s_i),p_{\theta}(s^{\mathcal{G}}_n))& & 
    \begin{aligned}
    &\text{if} \ \exists \ s_{t+m} \\
    &\text{achieves} \ s^{\mathcal{G}}_n,
    \end{aligned}\\
    &-\frac{1}{T}\sum_{i=t}^{t+T-1} \text{sim}(\rho_{\vartheta}(s_i),\rho_\vartheta(s_{t+T-1}))& & \text{otherwise}. 
\end{aligned}\right. 
\end{equation}

\section{Experiments}
In this section, we conduct a series of experiments to investigate the following issues: 

\begin{enumerate}
\item Whether GDCC can accurately identify states with high causal capacity and are these states suitable for use as subgoals in the environment?
\item Whether the prediction model is capable of accurately predicting the corresponding subgoals for each state?
\item Whether GDCC can effectively improve performance compared to baseline methods?
\item Is the time consumption of GDCC acceptable?
\end{enumerate}

We selected the MuJoCo-Maze \cite{2012MuJoCo} and Habitat \cite{Savva_Kadian_Maksymets_Zhao_Wijmans_Jain_Straub_Liu_Koltun_Malik_etal,Szot_Clegg_Undersander_Wijmans_Zhao_Turner_Maestre_Mukadam_Chaplot_Maksymets_etal,Puig_Undersander_Szot_Cote_Yang_Partsey_Desai_Clegg_Hlavac_Min_etal} environments as our benchmarks to evaluate the performance of GDCC. To increase the persuasiveness and effectiveness of the experiments, we modified the environments to provide sparse reward multi-objective tasks. In these tasks, the agent must navigate from a random starting point to a random endpoint. This modification increases the difficulty of the environment. In the sparse reward setting, the agent receives a non-zero reward only upon achieving the final goal. Instead of simply memorizing a path to complete the task, the agent must fully understand the dynamic changes within the environment and make reasonable decisions. The visualization of the Habitat environment is shown in Fig. \ref{fig:habitat}.

\begin{figure}
    \centering
    \includegraphics[width=0.7\columnwidth]{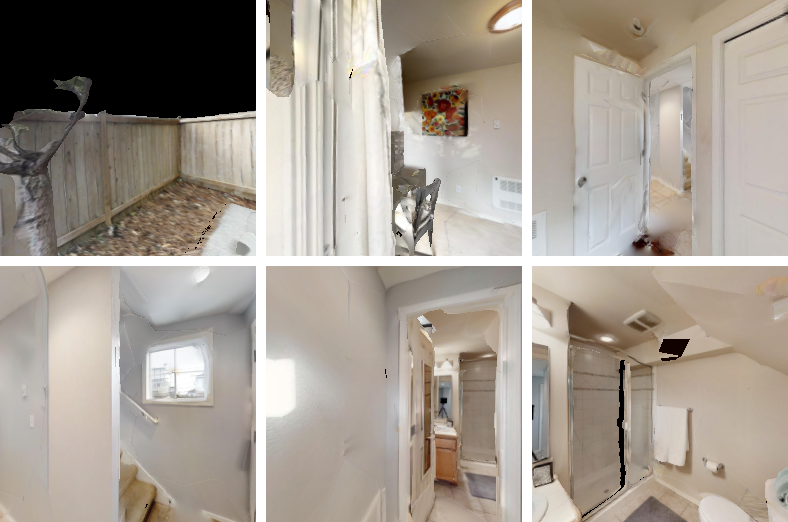}
    \caption{Visualization of the Habitat environment, which corresponds to the trajectory from the courtyard to the bathroom.}
    \label{fig:habitat}
\end{figure}

\begin{figure}[htb]
\centering
    \subfigure[]{
        \includegraphics[height=4.8cm]{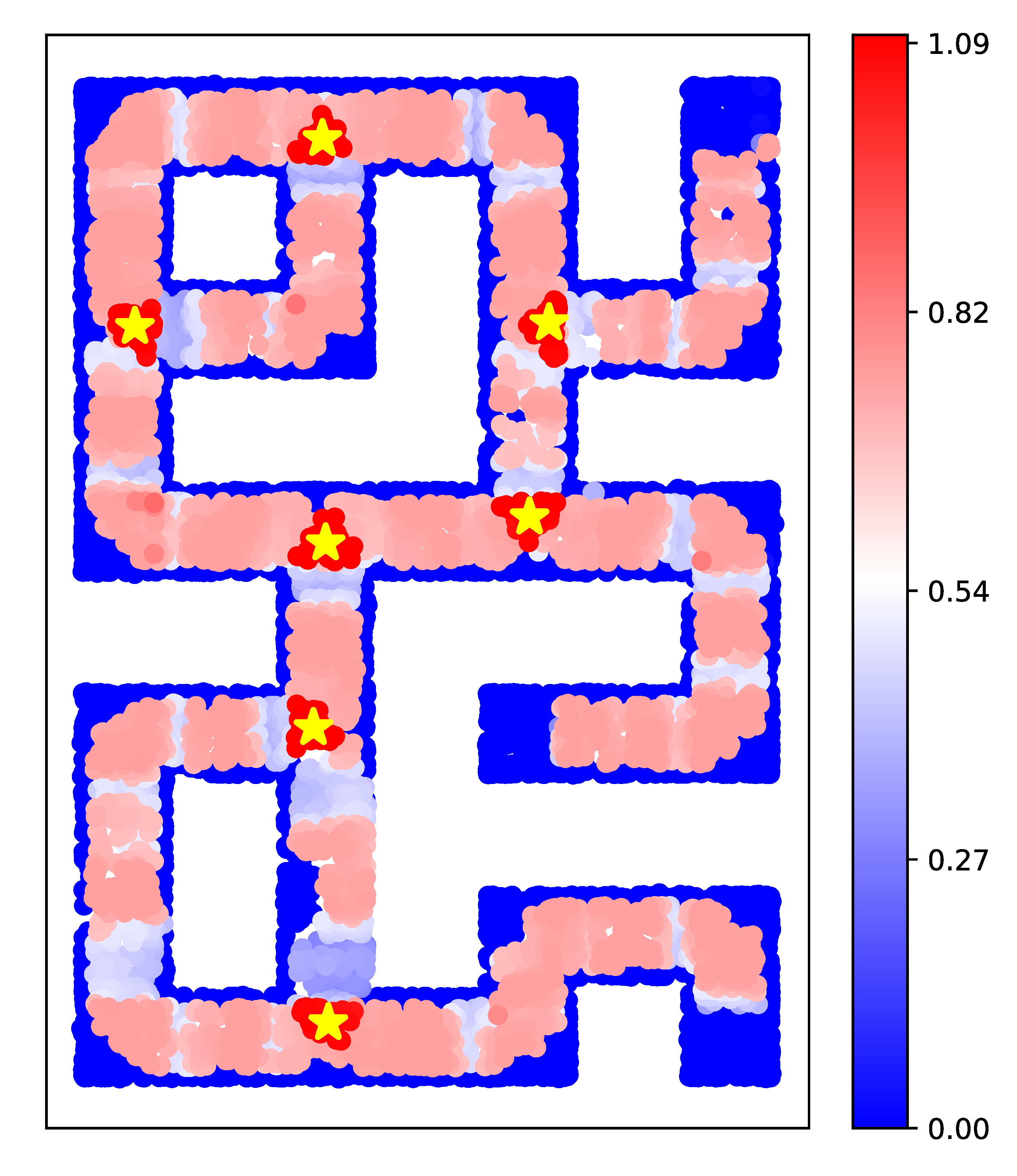}
        \label{fig:Maze causal capacity}
    }
    \subfigure[]{
        \includegraphics[height=4.8cm]{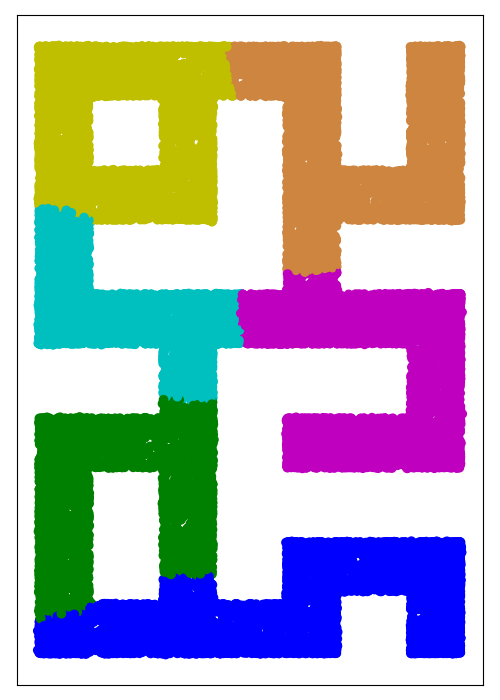}
        \label{fig:Maze prediction}
    }
    \subfigure[]{
        \includegraphics[height=4.8cm]{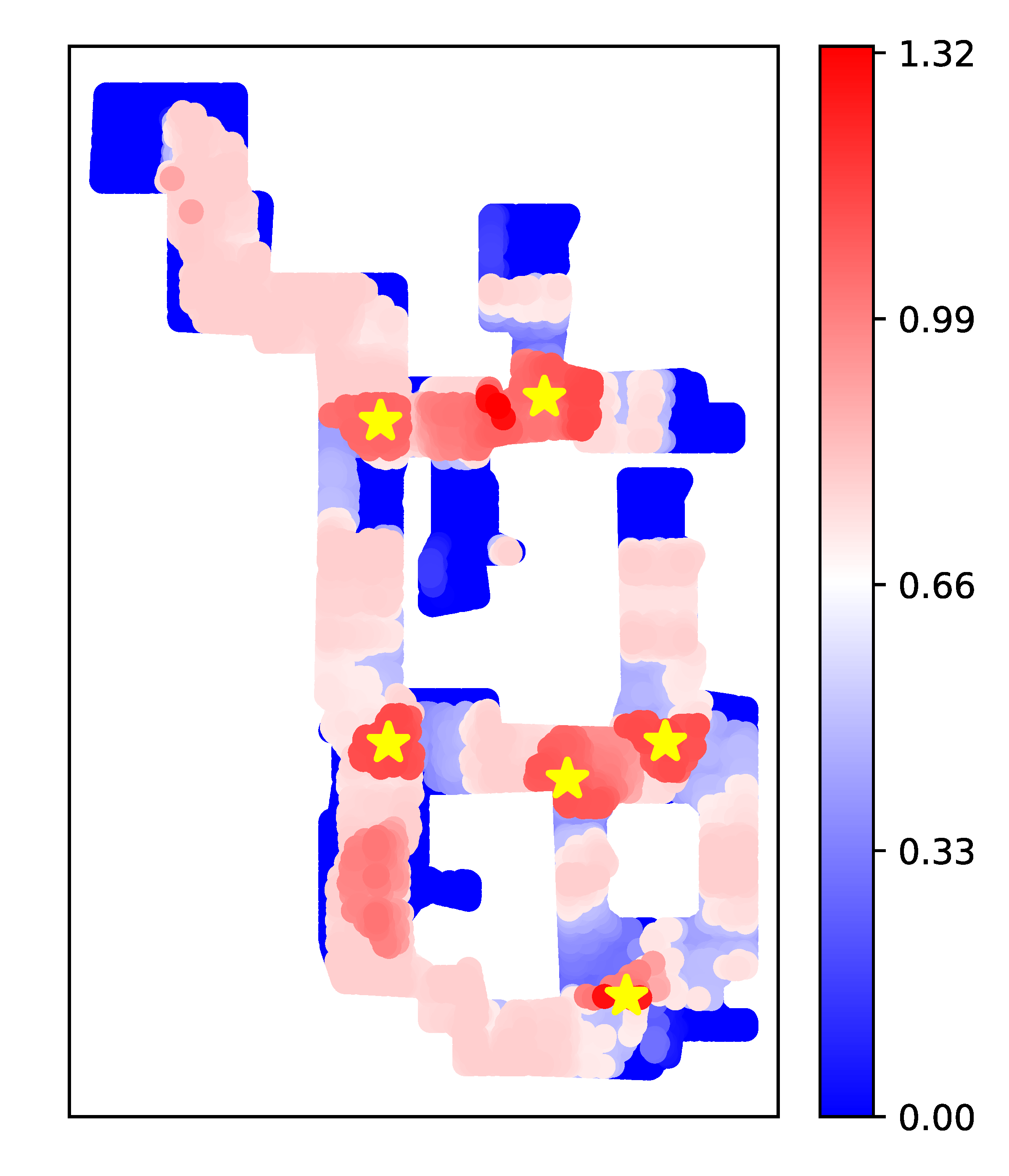}
        \label{fig:Annawan causal capacity}
    }
    \subfigure[]{
        \includegraphics[height=4.8cm]{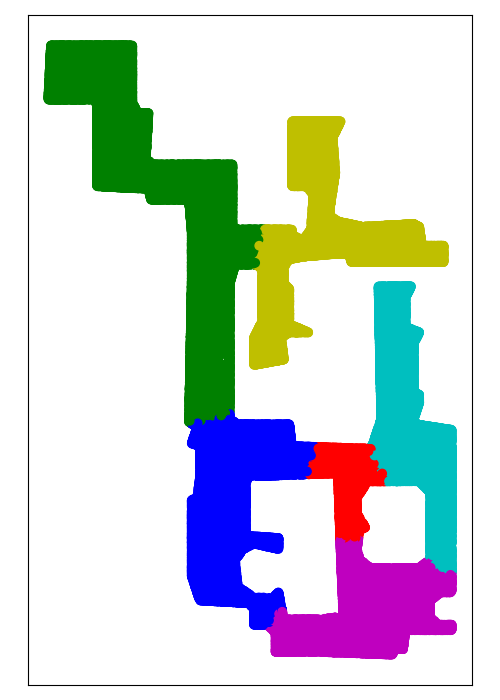}
        \label{fig:Annawan prediction}
    }
    \caption{Causal capacity calculation and subgoal prediction. (a) and (b): Results for Maze-large, (c) and (d): Results for Annawan. In the causal capacity calculation results, regions marked in red and blue indicate high and low causal capacity, while yellow stars represent the selected subgoals. In the subgoal prediction results, regions with the same color correspond to states predicted to the same subgoal by the GDCC.}
\label{result1}
\end{figure}

\subsection{Results of Causal Capacity Calculation}
We first evaluate the accuracy of the GDCC framework in estimating the causal capacity of each state in the environment. The results of causal capacity in the MuJoCo Maze-large and the Annawan and Applewold of Habitat are shown in Fig. \ref{result1}. High and low causal capacity states are represented by red and blue, respectively, with selected subgoals marked by yellow stars. The results demonstrate that GDCC accurately estimates the causal capacity of each state and selects subgoals that align with our expectations, effectively characterizing the intrinsic causality of the environment. More results can be found in the supplementary material.

\subsection{Results of Subgoal Prediction}
After calculating the causal capacity, we predict the corresponding subgoal for each state. This process can be seen as partitioning the state space into different regions. In Fig. \ref{result1}, we illustrate the partitioning results with different colors. Even in irregular maps like Habitat, the prediction model can clearly segment the boundaries of each region, ensuring the agent accurately acquires the optimal subgoal. More results are available in the supplementary material. 

Fig. \ref{fig:predictor} shows the curves of the reconstruction loss and the subgoal similarity loss during pretraining, as well as the curve of the predictor's accuracy. The encoder and decoder, trained with data sampled by the random policy, are able to quickly differentiate various subgoals while embedding states into latent space. The predictor can then accurately predict the optimal subgoal for the current state.

\begin{figure*}[htb]
    \begin{minipage}{0.28\textwidth}
    \centering
        \includegraphics[width=\columnwidth]{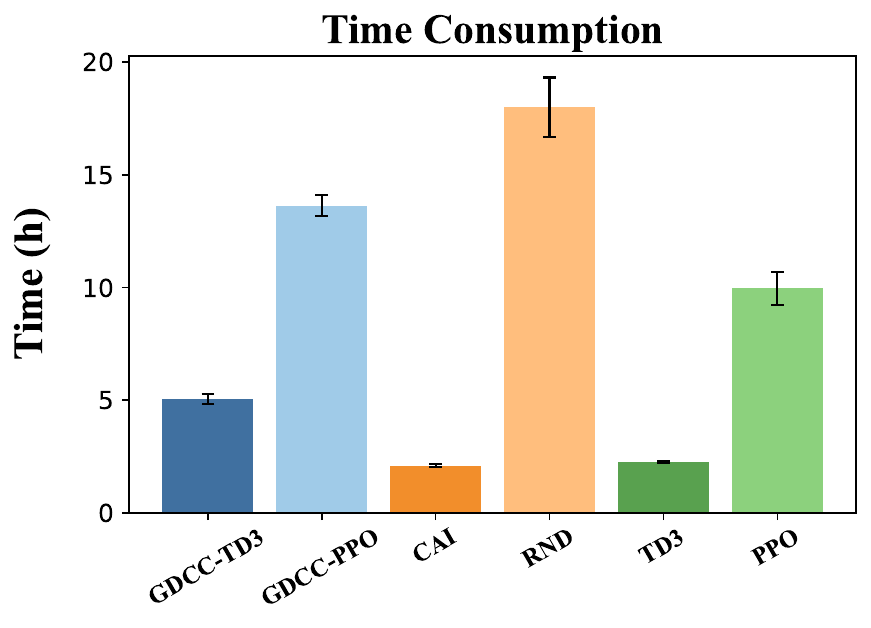}
        \caption{Time cost of each algorithm.}
        \label{fig:time}
    \end{minipage}
    \begin{minipage}{0.685\textwidth}
    \centering
        \includegraphics[width=\textwidth]{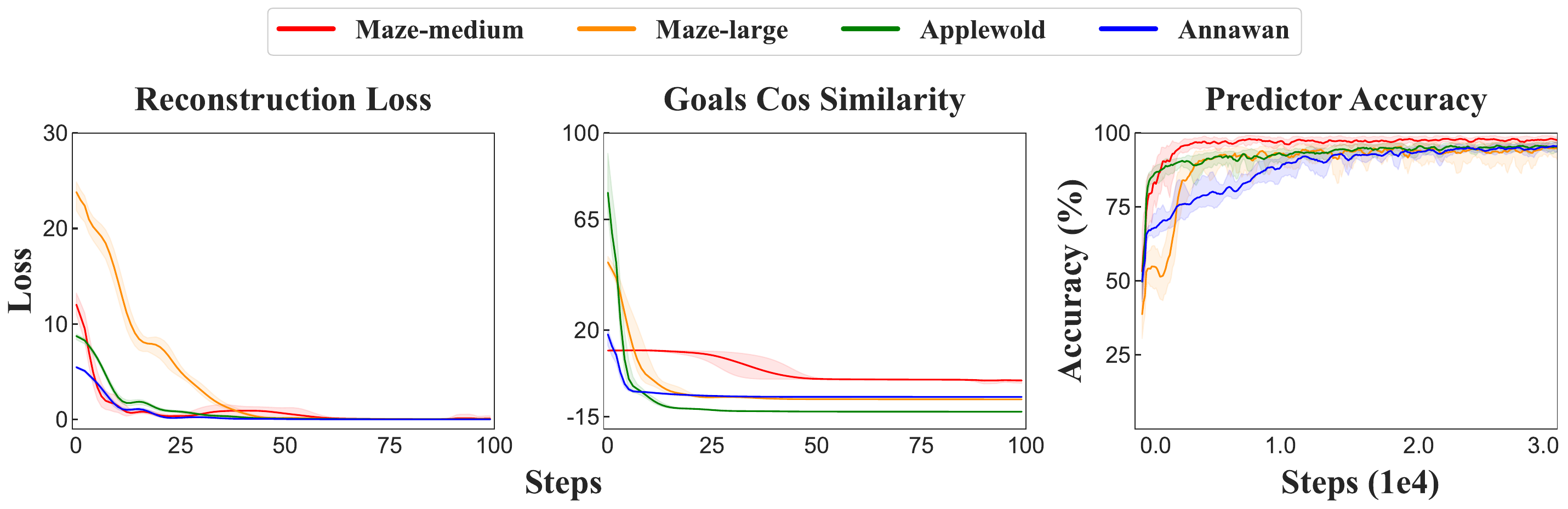}
        \caption{Illustration of the reconstruction loss (left), the cosine similarity of subgoals (middle), and the subgoal prediction accuracy (right).}
        \label{fig:predictor}
    \end{minipage}
    \end{figure*}

\begin{table}[ht]
    \centering
    \scriptsize
    \caption{Time consumption of each module in GDCC and each algorithm.}
    \label{table:time}
    \centering
    \begin{tabular}{lc|lc}
        \toprule
        \textbf{Module of GDCC}             & \textbf{Time(h)}             & \textbf{Algorithm}    & \textbf{Time(h)}\\ 
        \midrule
        Sampling Data               & $0.08\ \pm 0.01$        & GDCC-TD3              & $5.06\pm0.23$\\
        Calculating Causal Capacity & $0.02 \pm 0.04$          & GDCC-PPO              & $13.64\pm0.46$\\
        Training Subgoal Predictor  & $0.02 \pm 0.01$         & CAI                   & $2.09\pm0.03$\\
        Training TD3 Policy         & $4.94 \pm 0.24$     & RND                   & $18.01\pm 1.31$ \\
        Training PPO Policy         & $13.52 \pm 0.46$    & TD3                   & $2.26\pm 0.03$ \\
                                    &                           & PPO                   & $9.96 \pm 0.75$\\
        \bottomrule
        \end{tabular}
\end{table}

\begin{figure*}[htb]
    \centering 
    \includegraphics[width=0.85\textwidth]{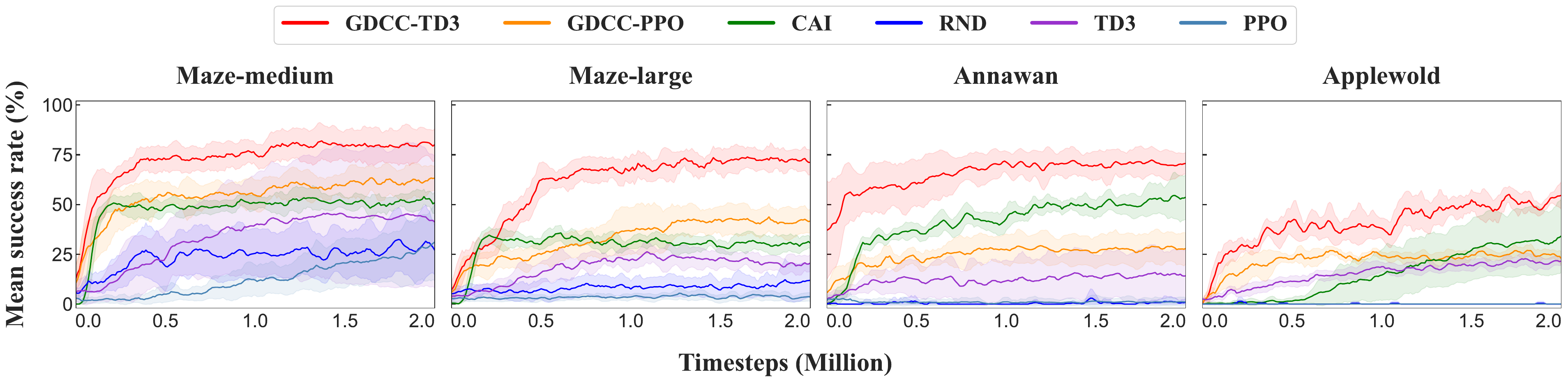}
    \caption{Comparison of our method against baselines in different scenarios: Maze-medium and Maze-large of MuJoCo, and Annawan and Applewold of Habitat.}
    \label{performance}
\end{figure*}
\subsection{Performance of GDCC}
To study the effectiveness of GDCC, we integrate it with two well-known reinforcement learning algorithms: Proximal Policy Optimization (PPO) \cite{Schulman_Wolski_Dhariwal_Radford_Klimov_2017} and Twin Delayed Deep Deterministic Policy Gradient (TD3) \cite{2018Addressing}, which are two major categories of RL. We then compare the performance of  the GDCC framework with several baseline algorithms, including Causal Action Influence (CAI) \cite{seitzer2021causal} and Random Network Distillation (RND) \cite{Burda_Edwards_Storkey_Klimov_2018} to demonstrate its effectiveness in both causal goal-conditioned RL and exploration RL.
For tasks with sparse rewards, we designed a potential-based reward function that only activates when subgoals are correctly predicted by GDCC. This design can have negative effects on other algorithms if no subgoal is utilized. More details regarding the design of the potential-based reward are presented in supplementary material. 

The empirical results presented in Fig. \ref{performance} clearly illustrate the superiority of GDCC. In various scenarios such as Maze-medium and Maze-large (MuJoCo environments) and Annawan and Applewold (Habitat environments), GDCC significantly outperforms the baselines. Specifically, the combination of GDCC and TD3 achieves at least a 25\% higher success rate on average than other algorithms. Although GDCC combined with PPO does not reach the highest performance, it still shows substantial improvements over PPO and RND alone. In the Habitat environment, where PPO and RND struggle to complete the tasks, the incorporation of GDCC leads to a notable increase in success rates.

\subsection{Ablation Study}
To investigate the contributions of each module in GDCC, we conducted an ablation study on the Maze-large environment. Fig. \ref{fig:ablation} demonstrates the improvements of subgoal predictor and potential-based reward to GDCC. The performance of GDCC is severely affected when the subgoal predictor is removed, the agent struggles to accomplish the task. This highlights the importance of correctly predicting the current subgoal for the hierarchical framework. The introduction of the potential-based reward helps GDCC explore more purposefully, enabling the agent to better understand both the environment and the task.

\begin{figure}[htb]
    \centering
    \includegraphics[width=0.99\columnwidth]{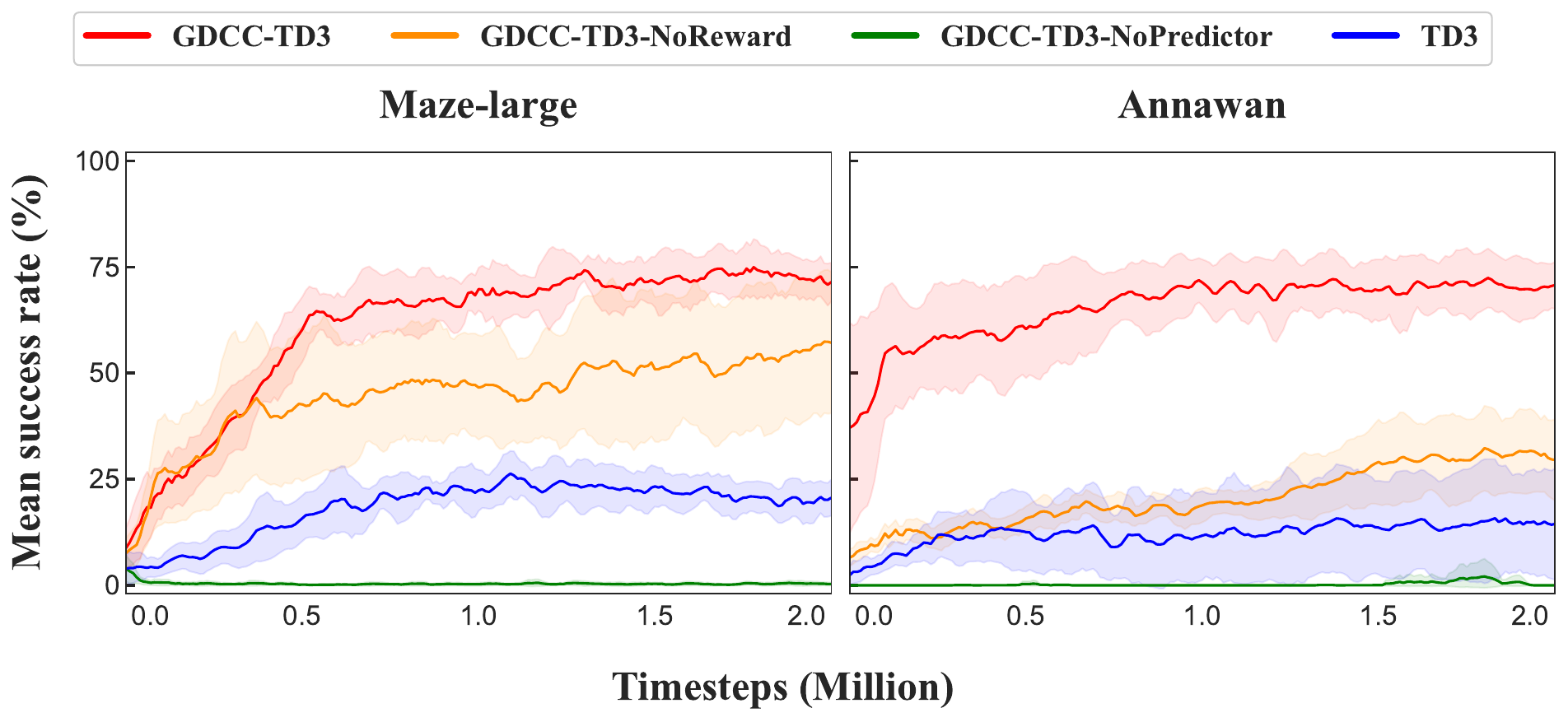}
    \caption{Ablation study results of GDCC. In the \textbf{NoReward} setting, the agent only receives the sparse reward from the environment. In the \textbf{NoPredictor} setting, the current subgoal is set to the closest subgoal in the state space.}
    \label{fig:ablation}
\end{figure}
\subsection{Computational Overhead}
 Fig. \ref{fig:time} and Tab. \ref{table:time} show the time consumption of each module in GDCC, including sampling data, calculating causal capacity and training the subgoal predictor. The time consumption of the pretraining phase is less than $3\%$ of the overall GDCC framework. We also present the time consumption for all of the baselines conducted on GeForce RTX 2080Ti.

\section{Conclusion}
In this paper, we propose the GDCC framework, which enables agents to actively explore the environment by understanding the causal influences of their actions on state transitions. 
By deriving causal capacity from Granger causality, we identify states where an agent's actions have the most significant impact on its future trajectory. 
Those critical states are selected as subgoals to guide the exploration of the agent.
Our empirical results demonstrate the efficacy of the GDCC framework in the MuJoCo and Habitat benchmarks, where GDCC consistently outperforms baselines.

\bibliographystyle{plainnat}
\bibliography{references}
\clearpage

\appendix
\section{Proofs of Proposition}\label{app:proof}

\begin{proof}[Proof of Prop.~\ref{prop:bound}]
For any variable $\mathbf{X}$, it holds that $p(\mathbf{X}=x) \in [0,1]$, and thus:
\begin{equation}
         \mathcal{H}(\mathbf{X})=-\sum_{x\in|\mathbf{X}|}p(x)\log p(x) \geq 0.
\end{equation}
\begin{equation}
     \mathcal{T}(\mathbf{X}\rightarrow \mathbf{Y}) = \mathcal{H}(\mathbf{Y}'|\mathbf{Y})-\mathcal{H}(\mathbf{Y}'|\mathbf{Y},\mathbf{X})\leq \mathcal{H}(\mathbf{Y}'|\mathbf{Y}) .
\end{equation}
This property applies to state and action variables. It proves that the upper bound of $\mathcal{T}(A\rightarrow S\mid {S=s,\text{do}(A=a)})$ (hereafter referred to as $\mathcal{T}_{A\rightarrow S}$) is:
\begin{equation}
    \begin{aligned}
     \mathcal{T}_{A\rightarrow S} &=\mathcal{H}(S'\mid S=s)-\mathcal{H}(S'\mid S=s, \text{do}(A=a))\\
	 &\leq \mathcal{H}(S'\mid S=s).
    \end{aligned}
\end{equation}

When the state transition is deterministic under a certain action, i.e.,  $\exists s_j,a_i,~p(s_j\mid s,a_i)=1$, we have: 

\begin{equation}
    \begin{aligned}
    \mathcal{H}(S'|S=s,\text{do}&(A=a_i))\\
    &= -p(s_j\mid s, a_i)\log p(s_j\mid s, a_i) = 0 .
    \end{aligned}
\end{equation}
So when the environment exhibits a deterministic state transition at $s$, the equality holds.

For the lower bound of the transfer entropy, we first factorize the non-interventional state transition probability $p(s \mid s)$ as follows:
\begin{equation}
p(s'\mid s) = \sum_{a\in\mathcal{A}}p(a\mid s)p(s'\mid s, a).
\end{equation}

We set $f(x)=-x\log x,~x\in[0,1]$. Obviously, $f(x)$ is a concave function and $\sum_{a \in \mathcal{A}} P(a \mid s) = 1$. According to Jensen's inequality, we have:
\begin{equation}
f(p(s\mid s))\geq \sum_{a\in\mathcal{A}}p(a\mid s)f(p(s'\mid s, a)).
\end{equation}
We further derive the non-interventional state transition entropy as follows:
\begin{equation}
    \begin{aligned}
	\mathcal{H}(S'\mid S&=s)  =-\sum_{s\in \mathcal{S}}p(s'\mid s)\log p(s'\mid s) \\
	&=\sum_{s\in\mathcal{S}}f(p(s'\mid s))\\
	&\geq \sum_{s\in\mathcal{S}}\sum_{a\in\mathcal{A}}p(a\mid s)f(p(s'\mid s, a))\\
	&=\sum_{a\in\mathcal{A}}p(a\mid s)\sum_{s\in\mathcal{S}}\left(-p(s'\mid s, a)\log p(s'\mid s, a)\right)\\
	&=\sum_{a\in\mathcal{A}}p(a\mid s)\mathcal{H}(S'\mid S=s, \text{do}(A=a)).\\
    \end{aligned}
    \label{eq:non}
\end{equation}
Since $p(a \mid s)\mathcal{H}(S' \mid S=s, \text{do}(A=a))\geq0, ~ \forall a\in\mathcal{A}$, then we have:
\begin{gather}
    \begin{aligned}
	\mathcal{H}(S'\mid S&=s)\\
        &\geq p(a\mid s)\mathcal{H}(S'\mid S=s, \text{do}(A=a)), ~\forall  a\in\mathcal{A},
    \end{aligned}\\
    \begin{aligned}
	\Rightarrow \mathcal{H}(S'\mid S=s, \text{do}&(A=a))\\
            &\leq \frac{1}{p(a\mid s)}\mathcal{H}(S'\mid S=s)~\forall  a\in\mathcal{A} .\label{eq:greater}
    \end{aligned}
\end{gather}

By substituting Eq. \ref{eq:greater} into the definition of transfer entropy, we can derive the lower bound of transfer entropy as follows:
\begin{equation}
    \begin{aligned}
	\mathcal{T}_{A\rightarrow S}&=\mathcal{H}(S'\mid S=s)-\mathcal{H}(S'\mid S=s, \text{do}(A=a))\\
	&\geq \min_{a}\left(1-\frac{1}{p(a\mid s)}\right)\mathcal{H}(S'\mid S=s).\\
    \end{aligned}
\end{equation}

In the context where the environment satisfies the condition of equal probability of taking each action under no intervention, the transfer entropy achieves its lower bound:
\begin{equation}
    \mathcal{T}_{A\rightarrow S}\geq (1-|\mathcal{A}|)\mathcal{H}(S'\mid S=s).
\end{equation}
\end{proof}

\begin{proof}[Proof of Prop.~\ref{prop:max-bound}]
As $\mathcal{T}_{A\rightarrow S}\leq \mathcal{H}(S'\mid S=s)$ has been proved above, the upper bound of $\max_a \mathcal{T}_{A\rightarrow S}$ is also $\mathcal{H}(S'\mid S=s)$.

Let $a_n$ be the action that achieves the minimum transfer entropy in state $s$:
\begin{equation}
\begin{aligned}
\mathcal{H}_{a_n}&=\mathcal{H}(S'|S=s, \text{do}(A=a_n))\\
&\leq\mathcal{H}(S'|S=s, \text{do}(A=a)), ~\forall a\in \mathcal{A}.
\end{aligned}
\end{equation}
According to Eq. \ref{eq:non}, we can further derive that:

\begin{equation}
\begin{aligned}
    \mathcal{H}(S'\mid S=s) &\geq\sum_{a\in\mathcal{A}}p(a\mid s)\mathcal{H}(S'\mid S=s, \text{do}(A=a))\\
    &\geq \sum_{a\in\mathcal{A}}p(a\mid s)\mathcal{H}_{a_n}=\mathcal{H}_{a_n}.
\end{aligned}
\end{equation}
Therefore, the lower bound of the transfer entropy is:

\begin{equation}
\setlength{\arraycolsep}{0pt}
\renewcommand{\arraystretch}{1.5}
\begin{array}{rll}

	\displaystyle\max_a\mathcal{T}_{A\rightarrow S}  &=\mathcal{H}(S'\mid S    & =s)\\
                                           &                         &-\displaystyle\min_a\mathcal{H}(S'\mid S=s, \text{do}(A=a))\\
	  \displaystyle                                    &= \mathcal{H}(S'\mid S   &=s)-\mathcal{H}_{a_n}\\
	  \displaystyle                                    &\geq\mathcal{H}(S'\mid S &=s)-\mathcal{H}(S'\mid S=s)\\
	                                    & = 0.                    &
\end{array}
\end{equation}

\end{proof}

\begin{proof}[Proof of Prop.~\ref{prop:random}]
For the non-interventional state transition probability $p(s'\mid s)$, it can be factorized based on the transition dynamics under each action \( a \in \mathcal{A} \) in the action space. 

\begin{equation}\label{eq:inf}
    p(s' \mid s) = \int_{\mathcal{A}}p(s' \mid s,a)p(a \mid s) da.
\end{equation}
Here, we consider the continuous action space. The proof is the same for discrete action space. Under the condition of non-intervention, we usually assume a uniform distribution over actions:
\begin{equation}\label{eq:equal}
    p(a \mid s) = \frac{1}{|\mathcal{A}|},~\forall a \in \mathcal{A}.
\end{equation}
 where $|\mathcal{A}|$ is the cardinality of the action space. Under the random policy, all actions also have an equal probability. Hence we can substitute Eq. \ref{eq:equal} into Eq. \ref{eq:inf}: 
\begin{align*}
p(s' \mid s) = & \int_{\mathcal{A}}p(s' \mid s,a) \frac{1}{|\mathcal{A}|} da\\
= & \int_{\mathcal{A}}p(s' \mid s,a) \pi_{\text{ran}}(a\mid s) da\\
= & p^{\text{do}(A:=\pi_{\text{ran}})}(s'\mid s).
\end{align*}

In some cases, due to characteristics of the environment, the non-interventional action distribution may follow a Gaussian distribution or another specific distribution, or some actions may rarely be executed. In such cases, we can adjust the random policy to better align with the action distribution of the environment. This adjustment enables us to use sampled data to estimate the non-interventional state transition distribution.
\end{proof}

\section{Clustering Algorithm Results}
In this section, we analyze the effectiveness of the clustering algorithm in estimating causal capacity across different scenarios. In Fig. \ref{demo maze split}, we present the results of calculating the causal capacity of the demo maze using the clustering algorithm, with specific positions selected to illustrate the specific details of the algorithm. In scenarios (b) and (c), where the agent encounters a crossroad and an endpoint, the clustering algorithm clearly captures the number and frequency of available choices, aligning with our expectations regarding causal capacity.
\begin{figure}[ht]
\centering 
    \includegraphics[width=\columnwidth]{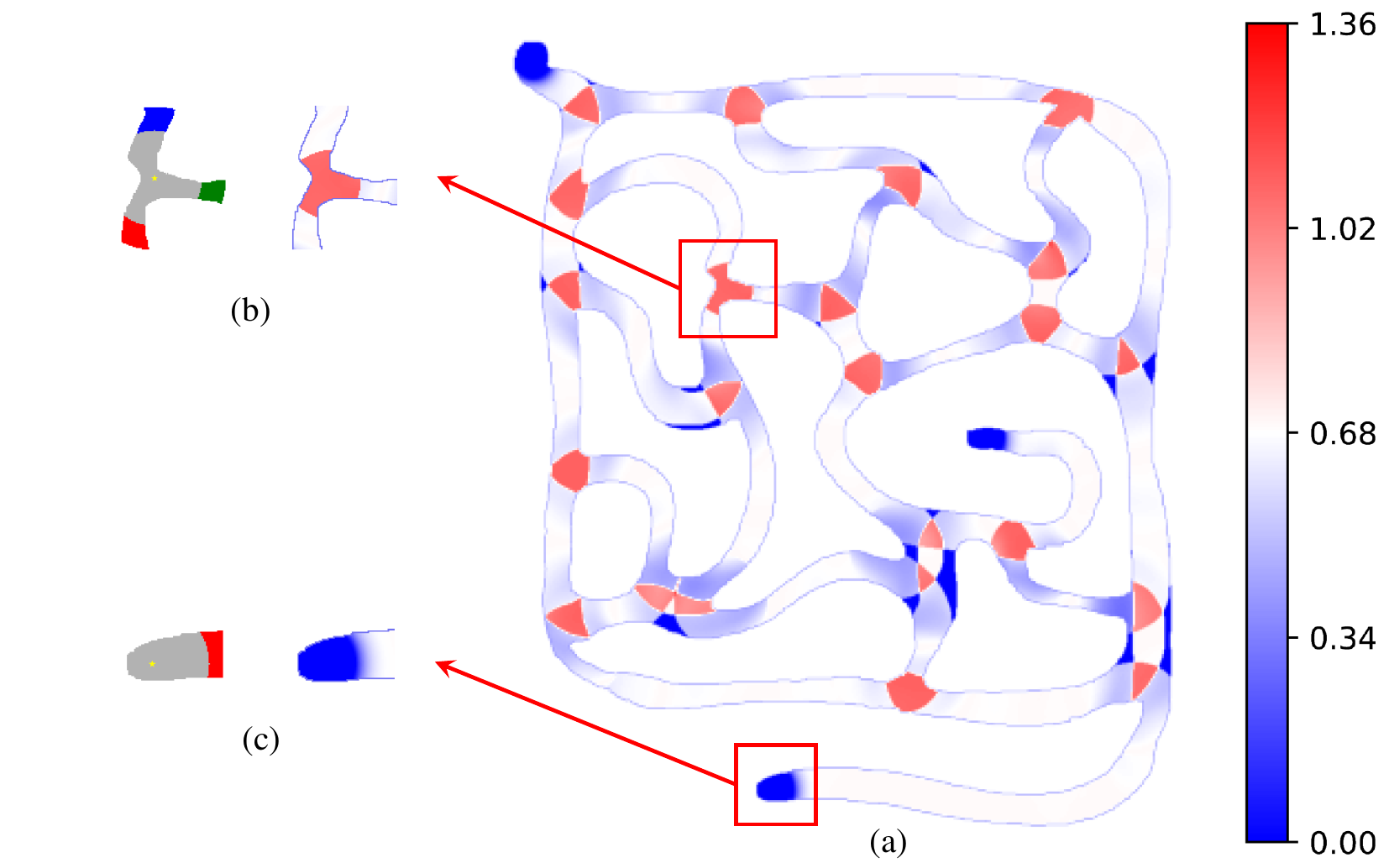}
    \caption{Illustration of the clustering algorithm in the demo maze. (a) Causal capacity of each state in the demo maze. (b) Clustering result at a three-way intersection, calculated based on the central state (marked with a star). (c) Clustering result at an endpoint.}
\label{demo maze split}
\end{figure}
In addition, we selected two special cases to further explain the results of the clustering algorithm. In Fig. \ref{fig.cl1} and Fig. \ref{fig.cl2}, when the agent is at a convergence point with a large area, the clustering algorithm is unable to identify the convergence point because it cannot predict a subsequent path within a single-step transition. This is due to the clustering threshold being designed based on the range of the agent's actions. This approach is practical, as for large convergence points, we are more concerned with the intersections rather than the convergence point itself. The approach is effective when the agent moves to the edge of the convergence point.

Besides convergence points, clustering algorithms can be effectively applied to various special situations. In an indoor environment, such as the one shown in Fig. \ref{fig.cl3}, clustering algorithms can accurately identify the entrances to each room. These entrances represent locations with the highest causal capacity in the indoor environment, offering clear guidance for the agent to achieve goals.

\begin{figure}[ht]
\centering 
\subfigure[]{
    \label{fig.cl1}
    \includegraphics[width=0.32\columnwidth]{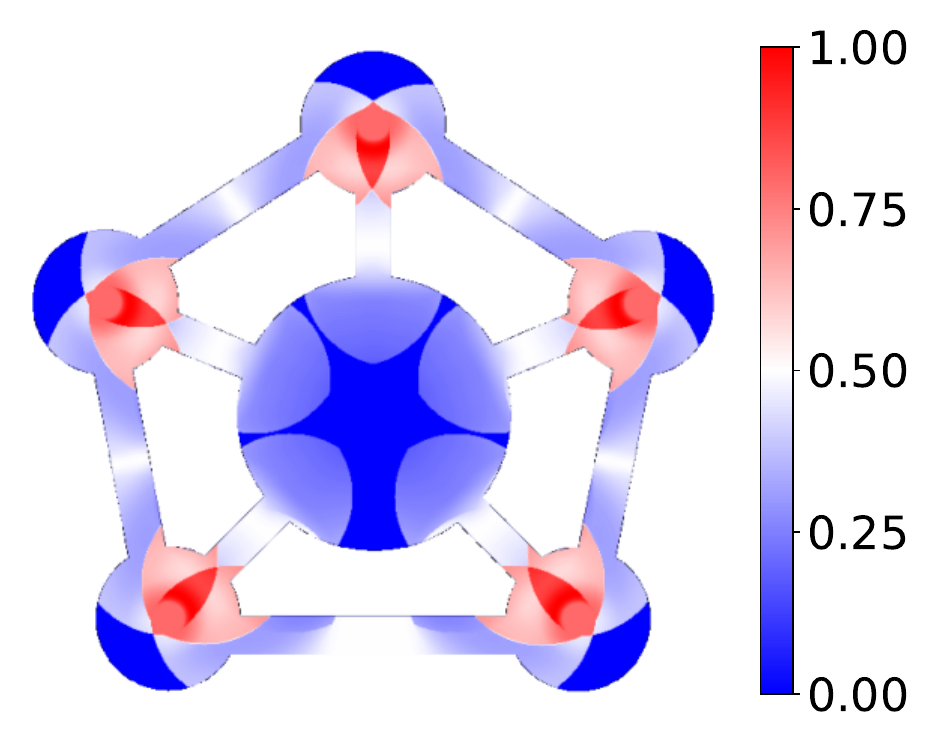}
}
\subfigure[]{
    \label{fig.cl2}
    \includegraphics[width=0.255\columnwidth]{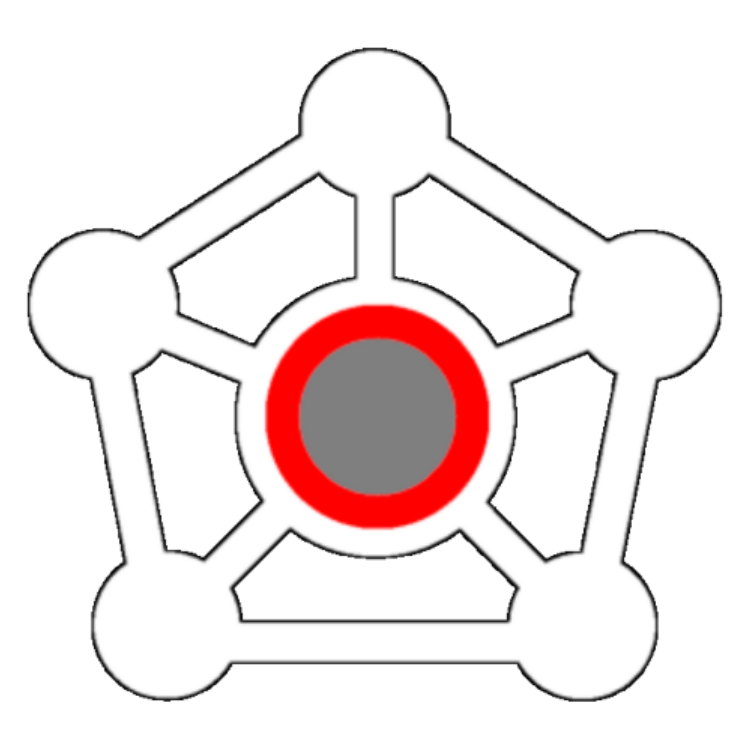}
}
\subfigure[]{
    \label{fig.cl3}
    \includegraphics[width=0.32\columnwidth]{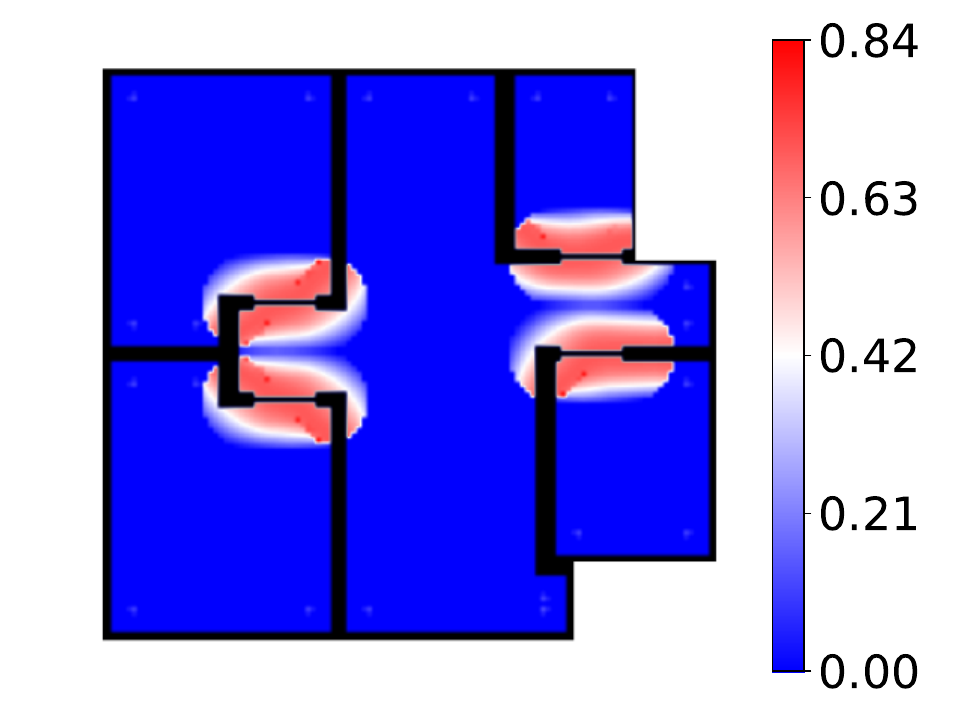}
}
\caption{(a) The causal capacity results in large convergence points. (b) The clustering results in large convergence points. If the agent's action does not lead to a clear state transition, it is considered as having no state transition from the causality perspective, resulting in a single cluster (the red region) around the agent. (c) The causal capacity results in an indoor environment. The entrances to each room are states with high causal capacity, representing the state transitions that occur when entering or exiting each room.}
\end{figure}
\label{app:reward}
In addition to identifying critical states where the agent's actions can determine its future, subgoals with high causal capacity offer the additional advantage in facilitating environment representation through the use of potential-based rewards. Potential-based rewards refer to the rewards the agent receives when transitioning from high-potential states to low-potential states, along with the equivalent penalty incurred when returning to high-potential states. This ensures that the optimal solution of the environment remains unchanged. However, a limitation arises when the agent has only a single final goal to achieve. In such cases, potential-based rewards may not sufficiently support the execution of complex behaviors, such as obstacle avoidance, turning, or ascending/descending stairs. This limitation is demonstrated in Fig. \ref{Potential-based reward}.

\begin{table*}[ht]
  \caption{Environment Settings}
  \label{table:environment}
  \centering
  \begin{tabular}{c|ccccc}
    \toprule
    Scenario  & $\tau_{\text{nei}}$ & $\tau_{\text{adj}}$ & Causal Capacity Threshold & Learning Rate & Episode Length \\
    \midrule
    Maze-medium   & $0.7$ & $1.0$ & $\log2.5$ & $1$e$-4$  & $600$ \\
    Maze-large    & $0.7$ & $1.0$ & $\log2.5$ & $1$e$-4$  & $600$ \\
    Annawan         & $0.8$ & $1.1$ & $\log2.8$ & $5$e$-5$  & $500$ \\
    Applewold       & $0.8$ & $1.1$ & $\log2.8$ & $5$e$-5$  & $500$ \\
    \bottomrule
  \end{tabular}
\end{table*}

With our prediction model, we can predict the corresponding subgoal for each state, allowing us to decompose the environment based on subgoals. Apart from the subgoal itself, there are no other states with high causal capacity within the subgoal's corresponding region, meaning the agent does not need to make complex decisions unless it is at a subgoal. As a result, the potential-based reward function within each subgoal's region is flat and effective. We only need to construct a potential-based reward function in each region and then concatenate them according to the transitions of each subgoal to create a comprehensive reward function that effectively represents the environment.
\begin{figure}[ht]
\centering 
\subfigure[Original rewards]{
    \label{PB.sub.1}
    \includegraphics[height=3cm]{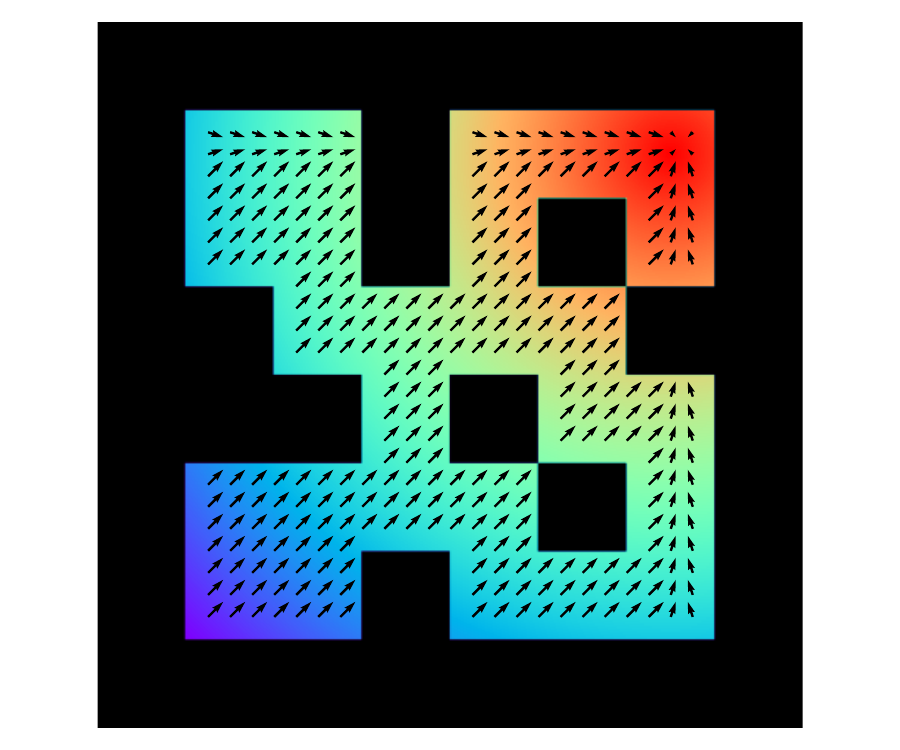}
}
\subfigure[Subgoals]{
    \label{PB.sub.2}
    \includegraphics[height=3cm]{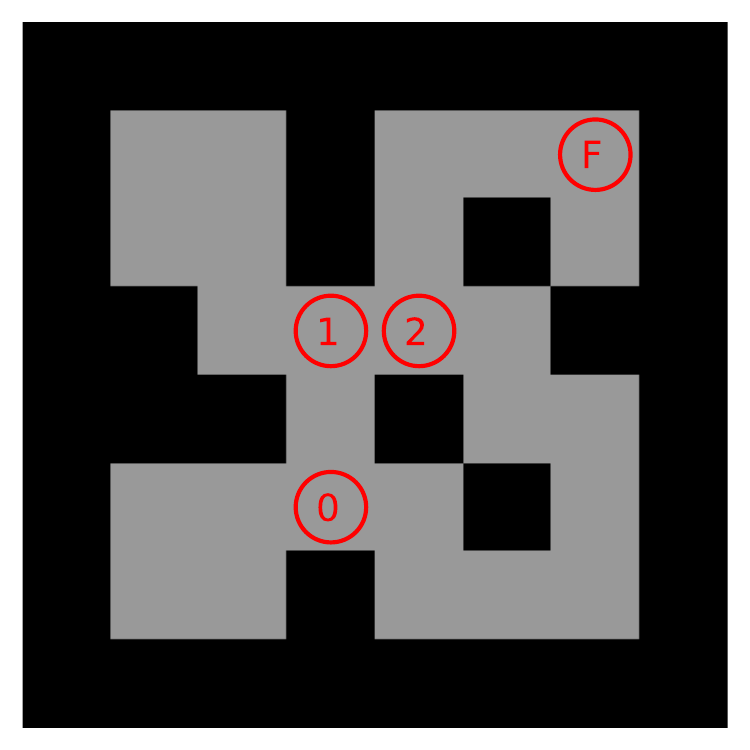}
}

\subfigure[Decomposed rewards]{
    \label{PB.sub.3}
    \includegraphics[height=3cm]{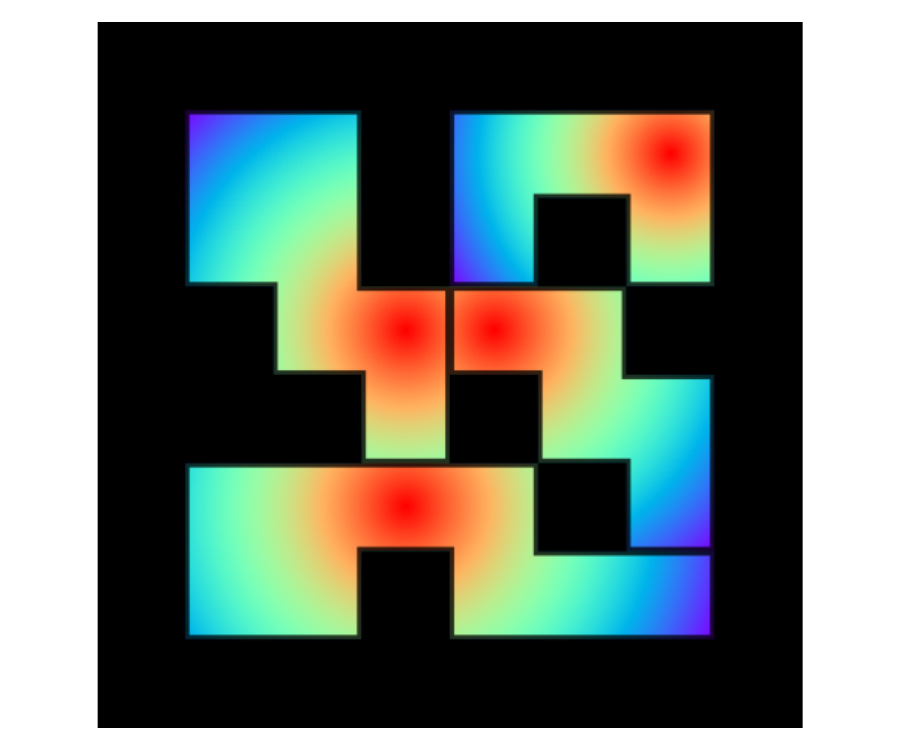}
}
\subfigure[Concatenated rewards]{
    \label{PB.sub.4}
    \includegraphics[height=3cm]{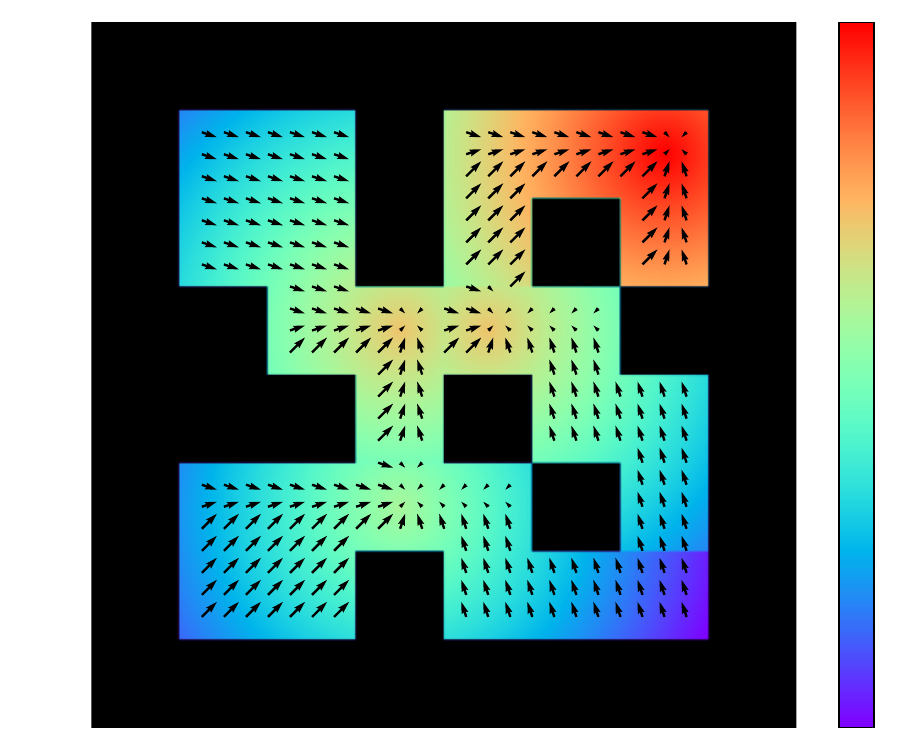}
}
\caption{(a) The original potential-based reward function, with the final goal located in the top-right corner of the environment. Arrows indicate the direction of the current reward gradient. It can be observed that this reward function fails to guide the agent past obstacles towards the final goal. (b) Subgoals calculated based on causal capacity. (c) The potential-based reward for each region after decomposing the state space according to each subgoal. (d) The concatenated reward function obtained by combining the potential-based reward functions of each region. It can be seen that this reward function effectively guides the agent around obstacles.}
\label{Potential-based reward}
\end{figure}

\section{Settings of Experiments}\label{app:settings}
\subsection{Hyperparameter Setting}
Here we present the hyperparameter settings for the experiments and environments in Tab. \ref{table:environment}-\ref{table:prediction}. The experiments were conducted on GeForce RTX 2080Ti and GeForce RTX 3070Ti GPUs.\label{app:gpu}
\subsection{Environments}
\paragraph{MuJoCo}
MuJoCo is a general-purpose physics engine designed for fast and accurate simulation of articulated structures interacting with their environment. It supports a wide range of models and environments, making it a popular benchmark for reinforcement learning experiments.

\begin{table}[ht]
    \centering
  \caption{General Settings}
    \label{table:general}
  \centering
  \begin{tabular}{ll}
    \toprule
    \textbf{Parameter}      & \textbf{Value}    \\ 
    \midrule
    Network Size            & $4 \times 256$    \\
    Gamma                   & $0.99$            \\
    Policy Noise            & $0.2$             \\
    Noise Clip              & $0.5$             \\
    Max Grad Norm           & $0.5$             \\
    Activation Function     & ReLU              \\
    Batch Size              & $1024$            \\
    Replay Buffer Size      & $200 000$         \\
    Replay Buffer Warmup    & $10 000 $         \\
    
    \bottomrule
  \end{tabular}
\end{table}

\begin{table}[htbp]
    \caption{Prediction Model Settings}
      \label{table:prediction}
      \centering
      \begin{tabular}{ll}
        \toprule
        \textbf{Parameter}          & \textbf{Value}    \\ 
        \midrule
        Network Size                & $3 \times 256$    \\
        Embedding Dimension         & $64$              \\
        Learning Rate               & $1\text{e}-3$     \\
        Similarity Function         & cos               \\
        Activation Function         & ReLU              \\
        Batch Size                  & $1000$            \\
        Encoder Training Times      & $10 000$          \\
        Predictor Training Times    & $10 000$          \\
        
        \bottomrule
      \end{tabular}
\end{table}
\paragraph{Habitat}
Habitat is designed for training agents to perform a variety of embodied AI tasks. For our experiments, we utilize the Gibson datasets \cite{xiazamirhe2018gibsonenv} within the Habitat simulator, which model real-world scenarios, including complex terrains such as furniture, rooms, and multi-story buildings. This environment closely mimics real-world settings, providing a more accurate reflection of an agent's ability to understand its environment. To compute the causal capacity of each state, we used the radar position information provided by the environment. Our experiments were conducted on two Habitat maps: Annawan, which features a single floor, and Applewold, a three-floor environment requiring traversal through stairs, making it more challenging. Fig. \ref{fig:3d} shows the 3D models of these two maps.
\begin{figure}[ht]
\centering 
\subfigure[Annawan]{
    \includegraphics[width=0.45\columnwidth]{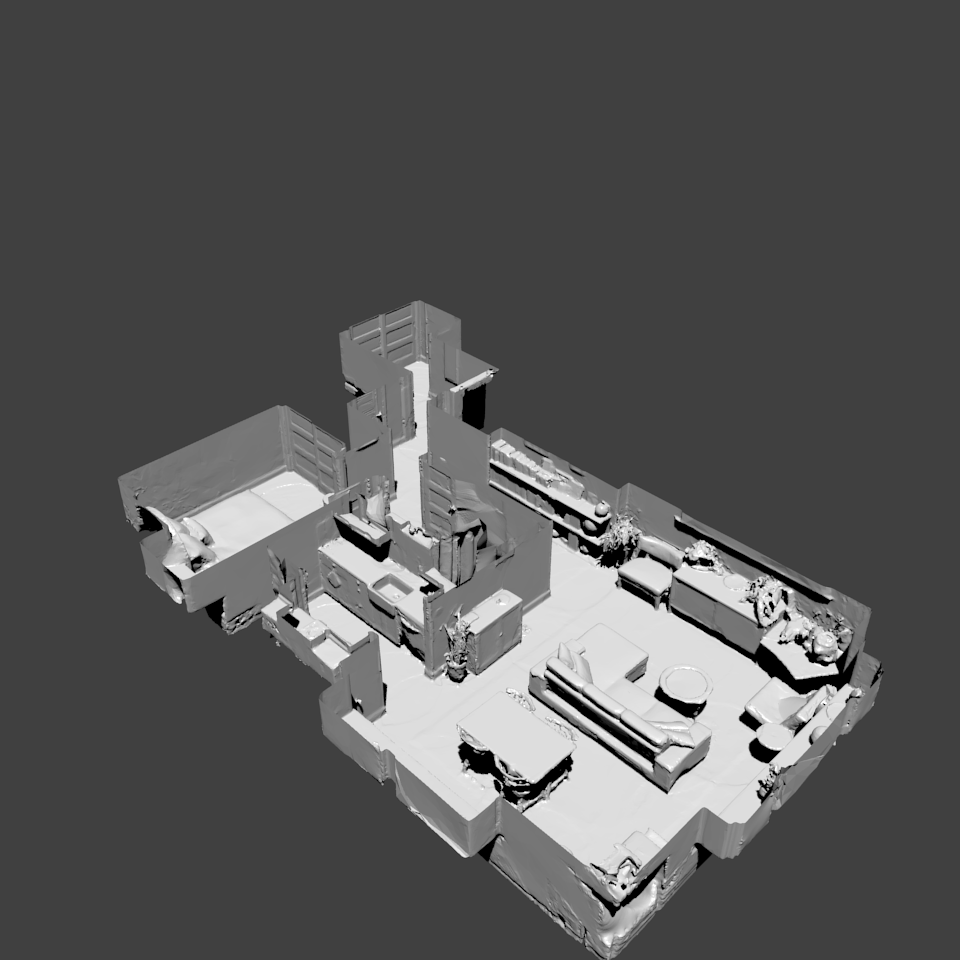}
}
\subfigure[Applewold]{
    \includegraphics[width=0.45\columnwidth]{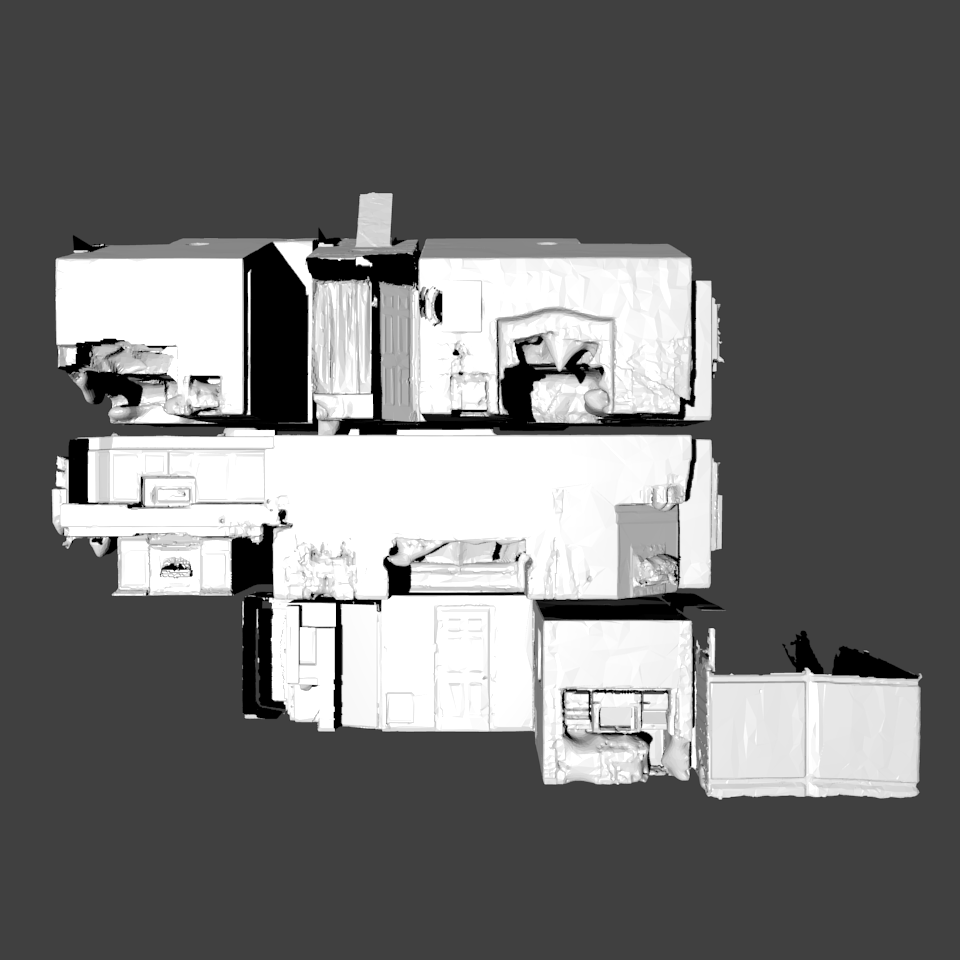}
}
\caption{3D models of two Habitat maps.}
\label{fig:3d}
\end{figure}

\subsection{Baselines}
We chose Proximal Policy Optimization (PPO) and Twin Delayed Deep Deterministic Policy Gradient (TD3) as the two basic algorithms of the GDCC framework. As baselines for comparison, we selected Causal Action Influence (CAI) and Random Network Distillation (RND). Both PPO and TD3 are widely used reinforcement learning algorithms known for their high sample efficiency and broad applicability. PPO is an improvement on the policy gradient method that stabilizes the learning process by limiting the magnitude of policy updates. TD3 is a deep reinforcement learning algorithm for continuous action spaces that addresses the instability of Deep Deterministic Policy Gradient (DDPG) \cite{2015Continuous} by using the delayed policy updates and twin-target Q-learning. CAI is a causal reinforcement learning method, integrates conditional mutual information into policy optimization to enhance the agent's understanding of how its actions influence the environment. RND promotes efficient exploration by providing intrinsic rewards, encouraging agents to discover novel states and actions that reduce environment uncertainty. 

\section{Remaining Results of GDCC}
\label{result}
To further verify the effectiveness of GDCC in estimating causal capacity and predicting subgoals, we present the results of causal capacity calculation and subgoal prediction  for the Maze-medium in Fig. \ref{fig:maze-medium} and Applewold and Capistrano maps from Habitat in Fig. \ref{fig:Applewold} and \ref{fig:Capistrano}.
Both Applewold and Capistrano are multi-story scenarios, we present the results for each floor, as well as the overall 3D predictions.
\begin{figure}
\centering 
\subfigure[Maze-medium causal capacity]{
    \label{EN.sub.medium}
    \includegraphics[height=3.7cm]{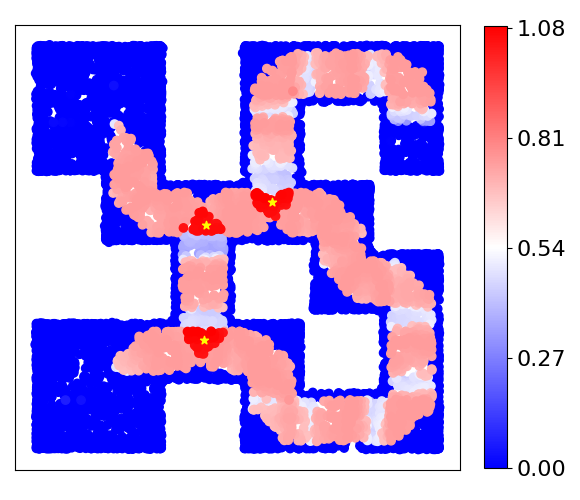}
}
\subfigure[Maze-medium prediction]{
    \label{DE.sub.medium}
    \includegraphics[height=3.7cm]{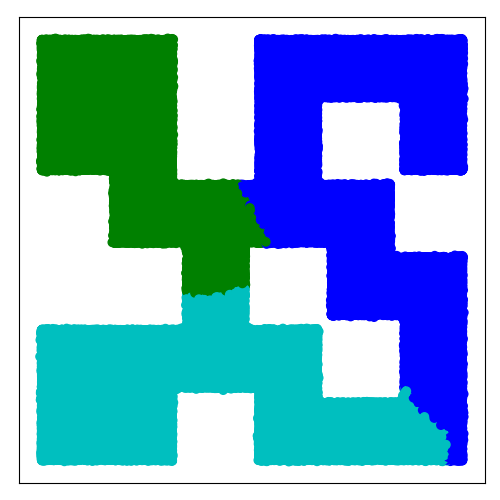}
}
\caption{Results of causal capacity calculation and subgoal prediction of GDCC for the Maze-medium.}
\label{fig:maze-medium}
\end{figure}

\begin{figure*}[htb]
\centering
\subfigure[]{
    \includegraphics[height=2.0cm]{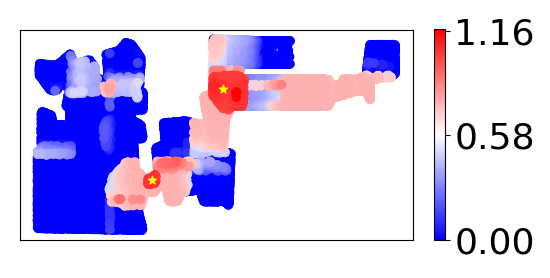}
        \label{fig:Applewold layer 1c}
}
\subfigure[]{
    \includegraphics[height=2.0cm]{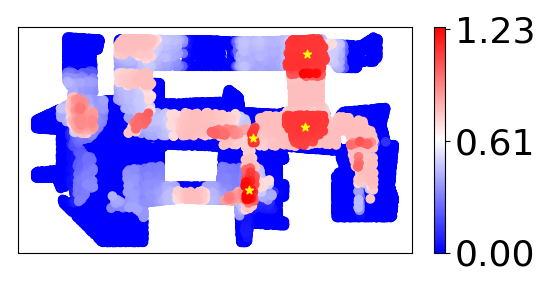}
        \label{fig:Applewold layer 2c}
}
\subfigure[]{
    \includegraphics[height=2.0cm]{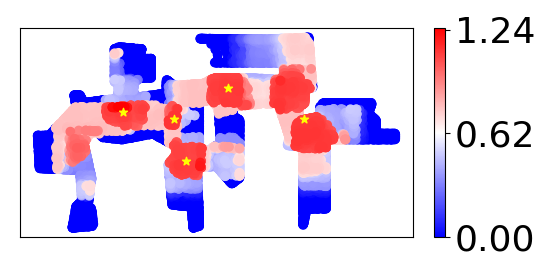}
        \label{fig:Applewold layer 3c}
}
\subfigure[Applewold 3d]{
    \includegraphics[height=2.3cm]{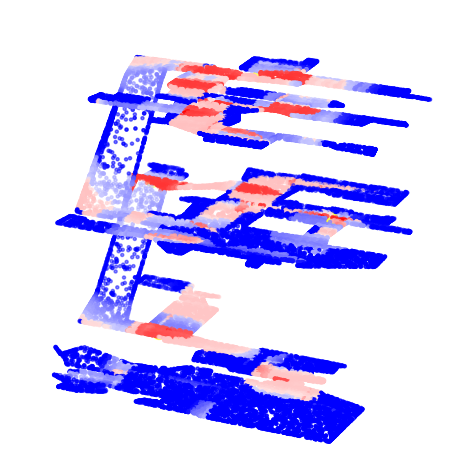}
        \label{fig:Applewold 3dc}
}

\subfigure[]{
    \includegraphics[height=2.0cm]{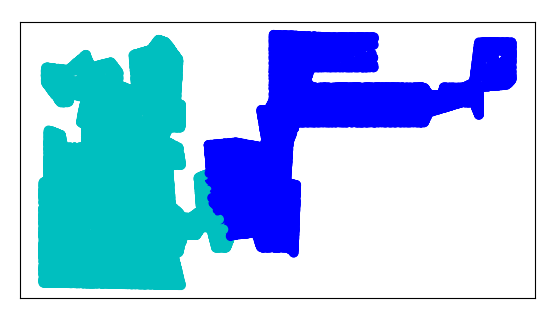}
        \label{fig:Applewold layer 1d}
}
\subfigure[]{
    \includegraphics[height=2.0cm]{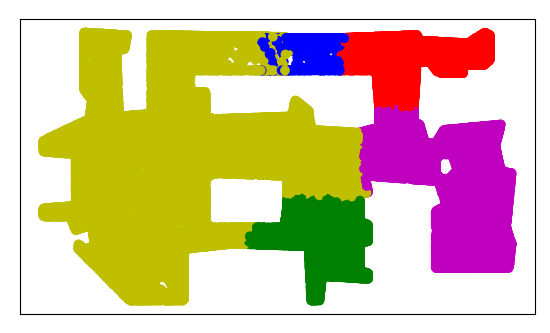}
        \label{fig:Applewold layer 2d}
}
\subfigure[]{
    \includegraphics[height=2.0cm]{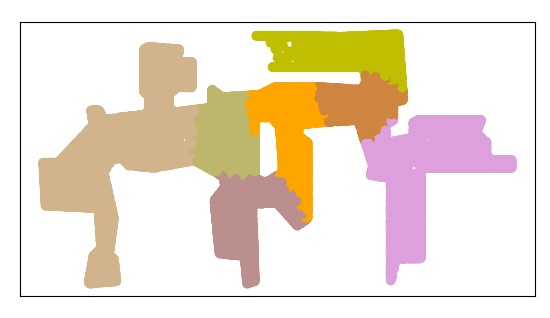}
        \label{fig:Applewold layer 3d}
}
\subfigure[Applewold 3d]{
    \includegraphics[height=2.3cm]{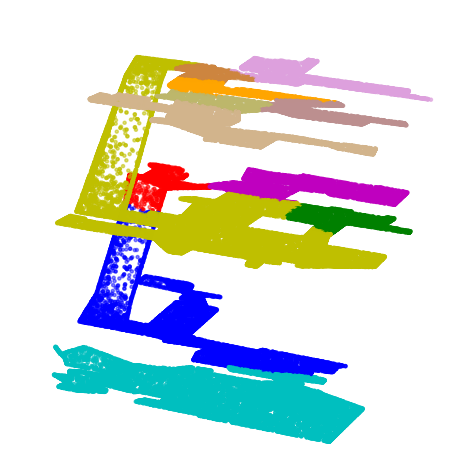}
        \label{fig:Applewold 3dd}
}
\caption{Results of causal capacity calculation \ref{fig:Applewold layer 1c}-\ref{fig:Applewold layer 3c} and subgoal prediction \ref{fig:Applewold layer 1d}-\ref{fig:Applewold layer 3d} of GDCC for each individual floor of Applewold map. \ref{fig:Applewold 3dc} and \ref{fig:Applewold 3dd} show the overall 3D predictions.}
\label{fig:Applewold}
\end{figure*}

\begin{figure*}[!ht]
\centering
\subfigure[Capistrano layer 1]{
    \label{fig:Capistrano layer 1dc}
    \includegraphics[height=3.2cm]{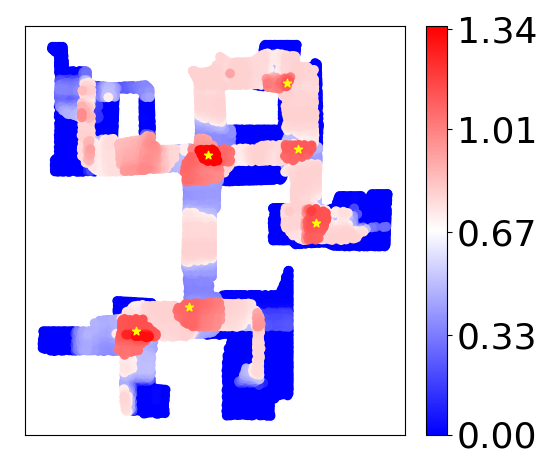}
}
\subfigure[Capistrano layer 2]{
    \label{fig:Capistrano layer 2dc}
    \includegraphics[height=3.2cm]{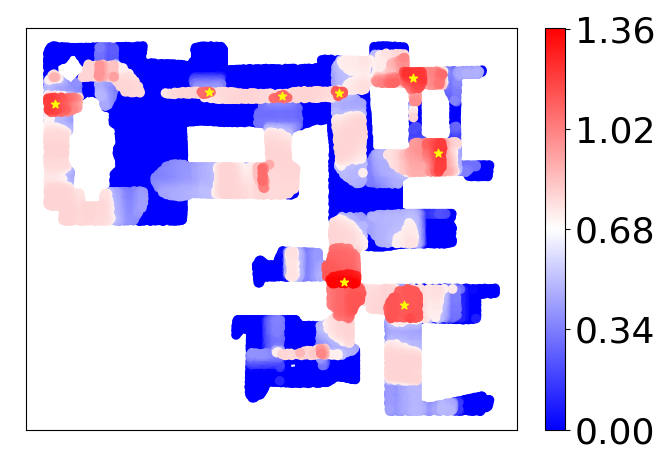}
}
\subfigure[Capistrano 3d]{
    \label{fig:Capistrano 3dc}
    \includegraphics[height=3.2cm]{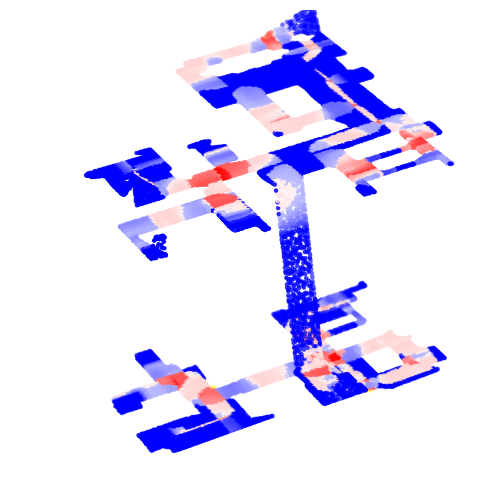}
}

\subfigure[Capistrano layer 1]{
    \label{fig:Capistrano layer 1dd}
    \includegraphics[height=3.5cm]{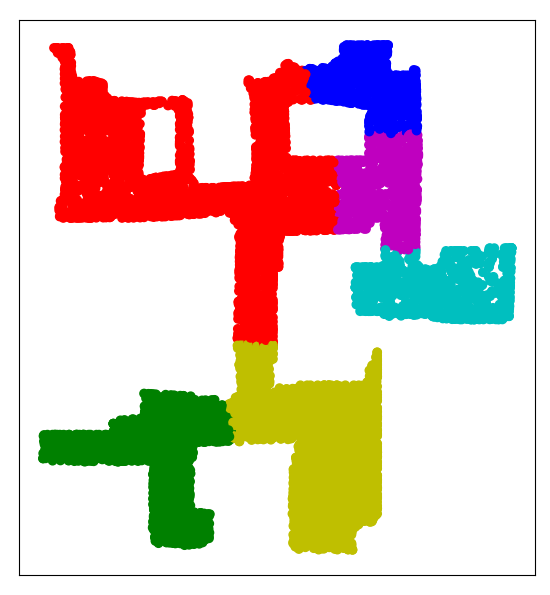}
}
\subfigure[Capistrano layer 2]{
    \label{fig:Capistrano layer 2dd}
    \includegraphics[height=3.5cm]{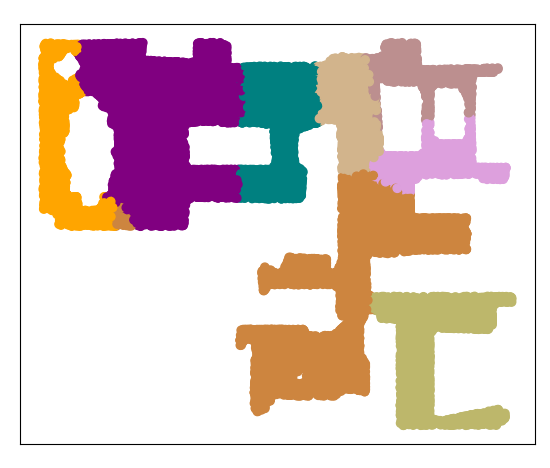}
}
\subfigure[Capistrano 3d]{
    \label{fig:Capistrano 3dd}
    \includegraphics[height=3.5cm]{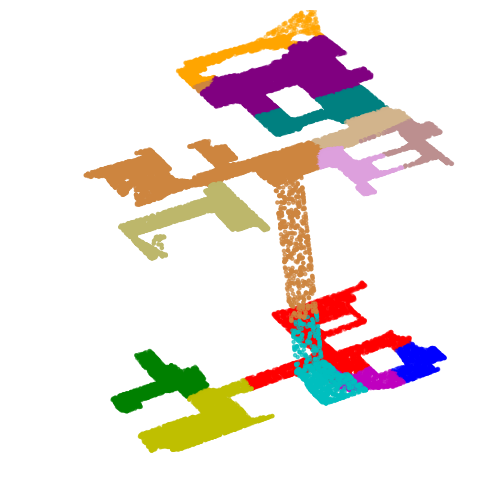}
}
\caption{Results of causal capacity calculation \ref{fig:Capistrano layer 1dc}-\ref{fig:Capistrano layer 2dc} and subgoal prediction \ref{fig:Capistrano layer 1dd}-\ref{fig:Capistrano layer 2dd} of GDCC for each individual floor of Capistrano map. \ref{fig:Capistrano 3dc} and \ref{fig:Capistrano 3dd} show the overall 3D predictions.}
\label{fig:Capistrano}
\end{figure*}
\end{document}